\documentclass[11pt]{article}

\usepackage[margin=1in]{geometry}

\usepackage{graphicx}%
\usepackage{multirow}%
\usepackage{amsmath,amssymb,amsfonts}%
\usepackage{amsthm}%
\usepackage{stfloats}
\usepackage{mathrsfs}%
\usepackage[title]{appendix}%
\usepackage{xcolor}%
\usepackage{textcomp}%
\usepackage{manyfoot}%
\usepackage{booktabs}%
\usepackage{algorithm}%
\usepackage{algorithmic}
\usepackage{listings}%
\usepackage{comment}
\usepackage{bm}
\usepackage{subfig}
\usepackage{hyperref}

\usepackage[sort]{natbib}

\usepackage{mathfont}
\usepackage{stmaryrd}
\usepackage[english]{babel}
\usepackage{graphicx}
\usepackage{textcomp}
\usepackage{multirow}
\usepackage{multicol}
\usepackage{bm}
\usepackage{mathtools}
\usepackage[makeroom]{cancel}
\usepackage{booktabs}       

\newcommand{\IfElse}[3]{
   \textbf{if}\ #1\ \textbf{then}\ #2\ \textbf{else}\ #3}
\newcommand{\red}[1]{{\color{red}{#1}}}

\DeclareMathOperator*{\argmin}{arg\,min}
\newcommand{\defeq}{\vcentcolon=}
\newcommand\norm[1]{\left\lVert#1\right\rVert}
\newcommand{\innerproduct}[2]{\langle #1, #2 \rangle}
\def\*#1{\mathbf{#1}}
\def\BibTeX{{\rm B\kern-.05em{\sc i\kern-.025em b}\kern-.08em
    T\kern-.1667em\lower.7ex\hbox{E}\kern-.125emX}}
\newcommand{\expnum}[2]{{#1}\mathrm{e}{#2}}

\theoremstyle{thmstyleone}%
\newtheorem{theorem}{Theorem}
\theoremstyle{thmstyletwo}%

\theoremstyle{thmstylethree}%
\newtheorem{definition}{Definition}%
\newtheorem{lemma}{Lemma}%
\title{LAMA-Net: A Convergent Network Architecture for Dual-Domain Reconstruction}

\author{Chi Ding\thanks{Department of Mathematics, University of Florida, Gainesville, FL 32601, USA. Email: \texttt{ding.chi@ufl.edu}} 
\and Qingchao Zhang\thanks{Department of Mathematics, University of Florida, Gainesville, FL 32601, USA. Email: \texttt{qingchaozhang3@gmail.com}} 
\and Ge Wang\thanks{Biomedical Imaging Center, Rensselaer Polytechnic Institute, Troy, NY 12180, USA. Email: \texttt{wangg6@rpi.edu}} 
\and Xiaojing Ye\thanks{Department of Mathematics and Statistics, Georgia State University, Atlanta, GA 30303, USA. Email: \texttt{xye@gsu.edu}} 
\and Yunmei Chen\thanks{Department of Mathematics, University of Florida, Gainesville, FL 32601, USA. Email: \texttt{yun@ufl.edu}}}

\date{}

\begin{document}

\maketitle

\begin{abstract}
We propose a learnable variational model that learns the features and leverages complementary information from both image and measurement domains for image reconstruction. In particular, we introduce a learned alternating minimization algorithm (LAMA) from our prior work, which tackles two-block nonconvex and nonsmooth optimization problems by incorporating a residual learning architecture in a proximal alternating framework. In this work, our goal is to provide a complete and rigorous convergence proof of LAMA and show that all accumulation points of a specified subsequence of LAMA must be Clarke stationary points of the problem. LAMA directly yields a highly interpretable neural network architecture called LAMA-Net. Notably, in addition to the results shown in our prior work, we demonstrate that the convergence property of LAMA yields outstanding stability and robustness of LAMA-Net in this work. We also show that the performance of LAMA-Net can be further improved by integrating a properly designed network that generates suitable initials, which we call iLAMA-Net. To evaluate LAMA-Net/iLAMA-Net, we conduct several experiments and compare them with several state-of-the-art methods on popular benchmark datasets for Sparse-View Computed Tomography.

\bigskip


\end{abstract}


\section{Introduction}\label{sec1}

Deep learning (DL) has transformed solution approaches in various fields, including computer vision and image processing, by efficiently learning features from training data and applying them to solve relevant problems in the past decades. Recently, DL techniques for tomographic image reconstruction have received significant attention. 
While most DL-based methods apply deep neural networks (DDNs), which is the core of DL, to the image domain for noise and artificial removal, there are limited works that leverage the power of DNNs in both image and measurement domains to improve solution quality in the literature. The latter method is to simultaneously recover image and full data, and thus naturally yields the problem called dual-domain image reconstruction.

In our prior work \cite{b2}, we introduced a deep network architecture called LAMA-Net for solving dual-domain image reconstruction. This network is naturally induced by the learnable alternating minimization algorithm (LAMA) developed in that work. In the present work, we provide a complete and rigorous mathematical proof of the convergence of LAMA.
More precisely: (i) We prove an important convergence property showing that any accumulation point of a specified subsequence of LAMA is a Clarke stationary point of the variational problem of the dual-domain reconstruction problem;
(ii) We show that the convergence property results in significantly improved efficiency, stability, and robustness of LAMA-Net. These improvements are justified by a number of new experiments and perturbated data sets; and (iii) We show that the performance of LAMA-Net can be further improved by using a properly designed network which generates suitable initials.

To demonstrate our main ideas, we focus on medical image reconstruction applications, particularly sparse-view computed tomography (SVCT) reconstruction. As an example, we also construct an initialization network (Init-Net), combined with LAMA-Net, to form an iLAMA-Net for solving SVCT. However, it is worth remarking that our methodology in LAMA-Net can be applied to general dual-domain reconstruction problems, and the Init-Net can be designed specifically for other applications.

The remainder of this paper is organized as follows. We present a review of recent work on image reconstruction in Section \ref{sec:relatedwork}. We provide a background of LAMA and LAMA-Net including their motivation and architecture in \ref{sec:background}. In Section \ref{sec:convergenceproof}, we conduct a complete convergence analysis of LAMA. Numerical experiments and results are given in Section \ref{sec:results}. Finally, Section \ref{sec:conclusion} concludes the paper.

\section{Related Work}
\label{sec:relatedwork}

\subsection{Deep learning in image reconstruction}

In recent years, DL has emerged as a powerful tool for image reconstruction. At the heart of DL are DNNs, a class of nonlinear parametric functions that can approximate highly complex functions with high accuracy \cite{b8,b9,b10}. In practice, DL-based methods use DNNs as surrogates of the functions of interest and learn the parameters of these DNNs using training data. The generalized data-driven learning framework enables DL to learn powerful high-level features from data for the target task. Thus, DL has produced promising empirical results in various applications.

More recently, unrolling deep networks \cite{b11,b12,b13,b14,b15} have become a major type of network architecture for image reconstruction. The design of an unrolling network is often inspired by some iterative optimization scheme and mimics the latter by treating one iteration as one layer (also called phase) of the network, i.e., the structure of the layer follows the computations in the iteration. But certain term(s) in that iteration, e.g., the proximal mapping, is replaced by a neural network with trainable parameters. The output of the unrolling network is set to be the result after a prescribed number of iterations (layers), and the parameters are trained by minimizing the error of the output to the ground truth image. 
The name ``unrolling'' thus refers to unfolding the iterative optimization scheme to form an end-to-end network. Optimization algorithms leveraging information in image domain \cite{b8,b13,b16,b17,b19,b20}, sinogram (data) domain \cite{b5,b22,b23}, and both domains \cite{b14,b24,b25,b26} have been proposed. Among these methods, dual-domain-based networks have been shown to outperform the others due to the advantage of adaptive feature selection and dynamic interactions in both domains \cite{b24,b25}. 
Therefore, alternating minimization (AM), also known as block coordinate descent (BCD), and similar algorithms become natural choices for reconstruction problems involving two domains. These algorithms are also designed to tackle optimization problems with two variables. 

Recent advances in dual-domain methodologies have demonstrated the value of cross-domain interactions in inverse problems. For instance, the Omni-Kernel Network \cite{liu2023omni} and Selective Frequency Network \cite{lin2023selective} propose innovative architectures for natural image restoration, leveraging frequency-aware mechanisms and kernel adaptation to enhance texture recovery. While these methods operate in photographic image domains, they highlight the broader potential of learned multi-domain priors—a perspective aligned with our dual-domain CT framework. Separately, Lukić and Balázs \cite{lukic2024moment, lukic2022limited, lukic2016binary} contribute rigorous model-based approaches to tomographic reconstruction, particularly for binary or moment-preserving scenarios. Their work on shape priors, centroid-driven reconstruction \cite{lukic2022limited} and geometric moment constraints \cite{lukic2024moment} offers valuable insights into non-data-driven regularization strategies, contrasting with our data-driven learned alternating minimization algorithm. While their focus on analytic shape descriptors differs fundamentally from LAMA’s deep learning integration, these works collectively underscore the importance of domain-specific prior assumptions in tomography, whether derived from mathematical principles or learned feature representations.
\subsection{Alternating Minimization and related unrolling networks}

AM and alike, including the Alternating Direction Method of Multipliers (ADMM) in a broad sense, have been well studied for multi-variate optimization, but mainly for convex problems, where the objective functions are jointly convex with respect to the variables. For non-convex problems, the proximal alternating linearized minimization (PALM) algorithm \cite{b27} is developed where the algorithm alternates between the steps of the proximal forward-backward scheme. A follow-up work \cite{b28} proposed an inertial version of the PALM (iPALM) algorithm motivated by the Heavy Ball method of Polyak \cite{b29}. It is also closely related to multi-step algorithms of Nesterov's accelerated gradient method \cite{b30}.
The global convergence to a stationary point is obtained for PALM and iPALM relying on the Kurdyka-\L{}ojasiewicz (KL) property. In addition, these two algorithms assume the joint term for two blocks to be smooth (possibly nonconvex), and the proximal points associated with these terms are trivial to obtain.

\subsection{Convergence and stability of unrolling networks}

Despite numerous empirical successes, unrolling methods also give rise to many critical concerns. First, these methods only superficially mimic existing iterative optimization schemes but lose the theoretical integrity of the regularization model. Second, they do not correspond to any specified objective function and, thus, do not possess any theoretical justifications or convergence guarantees under any specific optimality. 

In recent years, there has been a surge of interest in constructing convergent unrolling networks \cite{b34}.
There are mainly two types of network convergence considered in the literature: convergence in objective function value \cite{b15} and in fixed-point iteration (FPI) \cite{b35,b36,b37,b38}. Convergence of FPI leverages a fixed-point mapping parameterized as a deep network whose parameters $\theta$ can be learned from data.
Convergence can be established if the neural network's spectral radius falls in $[0,1)$ for any input imaging data, ensuring convergence due to the network's non-expansiveness. 
However, such an assumption may result in a large spectral radius and, thus, slow convergence in some medical imaging problems. In a specific class of FPI networks, training requires solving FPI and inverting operators using the Neumann series, which can be slow and unstable. Moreover, FPI networks do not link the limit points to any objective function, making it difficult to interpret the optimality of the solutions as in traditional variational methods in mathematical imaging. Therefore, our goal here is to construct a network that converges in function value. The details of this network architecture and its convergence property will be presented in the next sections.

\section{Background}
\label{sec:background}

\subsection{Learnable Dual-Domain Minimization Model}

\label{sec:model}

We first recall a dual-domain reconstruction model from our prior work \cite{b2}. 
Let $(\xbf,\zbf)$ be a pair of images and full data to be reconstructed, and $\*s_0$ is the given partially observed data. For example, $\xbf$ can be a CT image $\zbf$ is its corresponding full sinogram data, and $\*s_0$ is acquired sparse-view sinogram data. Then \cite{b2} considers the following minimization problem to recover $(\xbf,\zbf)$ from $\*s_0$:
\begin{equation}
\label{eq:OrgPhi}
    \min_{\xbf,\*z}\,\Phi(\xbf,\*z;\*s_0,\Theta)\defeq f(\xbf,\*z;\*s_0)+R(\xbf;\theta_1)+Q(\*z;\theta_2),
\end{equation}
where the data fidelity term $f$ is modeled by
\begin{equation}
\label{eq:datafidelity}
    f(\xbf,\*z;\*s_0)=\frac 1 2 \norm{\*{Ax}-\*z}^2 + \frac \lambda 2 \norm{\*P_{0}\*z - \*s_0}^2.
\end{equation}
For ease of interpretation, $\xbf$, $\zbf$, $\*s_0$ are all column vectors by stacking their original form (which are usually matrices) vertically.
Here the matrix $\Abf$ describes the physics of the transformation from image to data, for example, the Radon transform in CT applications, and $\*P_0 \in \mathbb{R}^{m\times n}$ ($m > n$) is a binary matrix that indicates how to select the specific measurement locations of the full data $\zbf$ to form the partial data $\*s_0$.
The term $f$ in \eqref{eq:datafidelity} is easy to understand: we want both consistencies between the image $\xbf$ and its full data $\zbf$ as well as between the full data $\zbf$ and its partially observed version $\*s_0$.

In \eqref{eq:OrgPhi}, there are two important parameters $\Theta:=(\theta_1,\theta_2)$ which are to be learned from data instead of being determined manually in most of the literature. 
We further note that $R(\xbf;\theta_1)$ and $Q(\*z;\theta_2)$ are the regularization terms in image and data domains, respectively. They are defined as follows:
\begin{subequations}
\label{eq:21norm}
\begin{eqnarray}
    R(\xbf, \theta_1)=\norm{\*g^R(\xbf,\theta_1)}_{2,1}&\defeq\sum_{i=1}^{m_R} \norm{\*g^R_i(\xbf,\theta_1)},\\
    Q(\*z,\theta_2)=\norm{\*g^Q(\*z, \theta_2)}_{2,1}&\defeq\sum_{j=1}^{m_Q} \norm{\*g^Q_j(\*z,\theta_2)}.
\end{eqnarray}
\end{subequations}
The $(2,1)$-norm is nonsmooth since the 2-norm of $g_i(x)$, given by $\|g_i(x)\|_2$, is not differentiable at $g_i(x) = 0$ because the gradient involves division by $\|g_i(x)\|_2$, which becomes undefined when $g_i(x) = 0$. We note that $\*g^R$ and $\*g^Q$ will be treated similarly, so we consider both in the same way as follows.
Let $r\in\{R,Q\}$, we have $\*g^r(\cdot)\in\mathbb{R}^{m_r\times d_r}$, where $d_r$ is the number of channels, and each channel is a feature map with size $\sqrt{m_r}\times\sqrt{m_r}$. Note $\*g^r_i(\cdot) \in \mathbb{R}^{d_r}$ is the vector at position $i$ across all channels. We apply a vanilla CNN as the feature extractor $\*g^r(\cdot)$:
\begin{equation}
    \label{eq:g_arch}
    \*g^r(\*y) = \*w_l * a \cdots a(\*w_2 * a(\*w_1 * \*y))
\end{equation}
where $\*w_i$ denotes the convolution parameters for each layer, and there are $l$ layers in total. $a(\cdot)$ is a smoothed version of the Rectified Linear Unit (ReLU) activation function, where ReLU$(x)=\max (0, x)$, and $\*y$ is the network input. The choice of the architecture of $\*g^r$ is rather flexible, but CNNs are generally considered more effective and lightweight than other types of DNNs for imaging problems. The choice of $(2,1)$-norm is to promote sparsity in order to pick up the most significant features extracted by $\*g^r$.
In \cite{b2}, a main contribution is to learn $ \*g^r(\*y)$ from data based on a novel learned alternating minimization algorithm (LAMA), or more precisely, its network form called LAMA-Net, to overcome drawbacks of manually designed regularization such as total-variation (TV) in dual-domain reconstruction.

\subsection{Learned Alternating Minimization Algorithm}
\label{sec:lama}

We note that \eqref{eq:OrgPhi} is a nonconvex nonsmooth minimization problem with respect to $\mathbf{x}$ and $\mathbf{z}$. This is a significant challenge to be overcome by the learned alternating minimization algorithm (LAMA) proposed in \cite{b2}.
Before we provide a complete and rigorous proof of the convergence of LAMA in Section \ref{sec:convergenceproof}, which is the main goal of this paper, we first quickly review the structure of LAMA.
Since $\Theta$ is fixed In the development of LAMA, we drop $\Theta$ for notation simplicity.

The first stage of LAMA aims to mollify the nonconvex and nonsmooth objective function in \eqref{eq:OrgPhi} to a nonconvex but smooth objective function by using the smoothing procedure \cite{b15,b40}: 
\begin{equation}
    \label{eq:smoothedR}
    \begin{split}
        r_\varepsilon(\*y) = \sum_{i=1}^{m_r} r_{\varepsilon,i}(\*y)=:
        \sum_{i\in I^r_0}\frac{1}{2\varepsilon}\norm{\*g^r_i(\*y)}^2 + \sum_{i\in I_1^{r}}\left( \norm{\*g^r_i(\*y)}-\frac \varepsilon 2\right),\quad\ybf \in Y,
    \end{split}
\end{equation}
where $(r,\*y)$ represents either $(R, \xbf)$ or $(Q, \*z)$, $Y:=\mathbb{R}^{m_r \times d_r}$, and
\begin{equation}
    I_0^r=\{i\in \left[m_r\right]\,\vert\norm{\*g^r_i}\leq \varepsilon\},\quad I_1^r = \left[m_r\right] \setminus I_0^r.
\end{equation}
It is easy to show:
\begin{equation}
    \label{eq:gradsmoothedR}
    \begin{split}
        \nabla r_\varepsilon(\*y) =\sum_{i\in I^r_0}\nabla \*g_i^r(\*y)^\top\frac{\*g^r_i(\*y)}{\varepsilon} + \sum_{i\in I_1^{r}}\nabla\*g^r_i(\*y)^\top\frac{\*g^r_i(\*y)}{\norm{\*g^r_i(\*y)}}, \quad \ybf \in Y.
    \end{split}
\end{equation}

The second stage of LAMA is to solve the smoothed nonconvex problem with the fixed smoothing factor $\varepsilon$, i.e.,
\begin{equation}
 \label{SM}   
\min_{\xbf, \*z}\{\Phi_{\varepsilon}(\xbf, \*z):=f(\xbf,\*z)+R_{\varepsilon}(\xbf)+Q_{\varepsilon}(\*z)\}.
\end{equation}
Despite the parameter $\varepsilon$ being fixed, which makes \eqref{eq:OrgPhi} and \eqref{SM} different, the next stage of LAMA has a mechanism to automatically decrease $\varepsilon$ so the solution to \eqref{SM} will converge to the original problem \eqref{eq:OrgPhi}. This will be proved rigorously in Section \ref{sec:convergenceproof}.
For now, we are still in the second stage and assume $\varepsilon = \varepsilon_k$ for some iteration $k$ of LAMA, and then we adopt the idea of the proximal alternating linearized minimization (PALM) \cite{b27}:
\begin{subequations}
    \begin{align}
    \label{eq:zgradstep}
    \*b_{k+1} &=\*z_k-\alpha_k \nabla_{\*z} f(\xbf_k,\*z_k)\\
    \label{eq:ylinprox}
    \*u_{k+1}^{\*z} &=\argmin _{\*u}  \frac{1}{2\alpha_k}\norm{\*u-\*b_{k+1}}^2+Q_{\varepsilon_k}(\*u)
    \end{align}
\end{subequations}
and
\begin{subequations}
    \begin{align}
    \label{eq:ck1y}
        \*c_{k+1} &=\xbf_k-\beta_k \nabla_{\xbf} f(\xbf_k,\*u^{\*z}_{k+1})\\
        \*u_{k+1}^{\xbf} &= \argmin_{\*u} \frac{1}{2\beta_k}\norm{\*u-\*c_{k+1}}^2+R_{\varepsilon_k}(\*u)
    \end{align}
\end{subequations}
where $\alpha_k$ and $\beta_k$ are the corresponding step sizes. 

Since the proximal point $\*u^{\xbf}_{k+1}$ and $\*u^{\*z}_{k+1}$ are difficult to compute in general due to the complicated structures of $Q_{\varepsilon_k}(\*u)$ and $R_{\varepsilon_k}(\*u)$, we approximate them by their linear approximations at $\*b_{k+1}$ and $\*c_{k+1}$ respectively, i.e., $Q_{\varepsilon_k}(\*b_{k+1})+ \innerproduct{\nabla Q_{\varepsilon_k}(\*b_{k+1})}{ \*u-\*b_{k+1}}$ and $R_{\varepsilon_k}(\*c_{k+1})+ \innerproduct{\nabla R_{\varepsilon_k}(\*c_{k+1})}{ \*u-\*c_{k+1}}$, plus simple quadratic terms $\frac{1}{2p_k} \norm{\*u-\*b_{k+1}}^2$ and $\frac{1}{2q_k} \norm{\*u-\*c_{k+1}}^2$. Then by a simple computation, $\*u^{\xbf}_{k+1}$ and $\*u^{\*z}_{k+1}$ are now determined as follows:
\begin{equation}
\label{eq:u}
        \*u^{\*z}_{k+1} = \*b_{k+1}-\hat{\alpha}_k\nabla Q_{\varepsilon_k}(\*b_{k+1}), \quad
        \*u^{\xbf}_{k+1} = \*c_{k+1}-\hat\beta_k\nabla R_{\varepsilon_k}(\*c_{k+1}),
\end{equation}
where $\hat\alpha_k = \frac{\alpha_k p_k}{\alpha_k + p_k}$, $\hat\beta_k = \frac{\beta_k q_k}{\beta_k + q_k}$. In the training of $\Theta$, we can also include $\alpha_k$, $\hat\alpha_k$, $\beta_k$, and $\hat\beta_k$ for training so they can act effectively in LAMA-Net, which usually runs for several steps in practice. However, to ensure convergence of LAMA, we need certain conditions on these steps, which will be specified later. 

Indeed, we note that even the convergence of the sequence $\{(\*u^{\*z}_{k+1}, \*u^{\xbf}_{k+1})\}_k$ is not guaranteed now. 
To ensure their convergence, we propose the following accept-or-reject procedure: if $(\*u^{\*z}_{k+1}, \*u^{\xbf}_{k+1})$ satisfy the following sufficient descent conditions (SDC):
\begin{subequations}
\label{eq:uvcond}
    \begin{align}
    \label{eq:uvconda}
        \Phi_{\varepsilon_k}(\*u^{\xbf}_{k+1},\*u^{\*z}_{k+1})-\Phi_{\varepsilon_k}(\xbf_k,\*z_k) &\leq 
         -\eta \left(\norm{\*u^{\xbf}_{k+1}-\xbf_k}^2+\norm{\*u^{\*z}_{k+1}-\*z_k}^2\right) \\
    \label{eq:uvcondb}
        \norm{\nabla_{\xbf,\zbf} \Phi_{\varepsilon_k}(\xbf_{k},\*z_{k})}&\leq \frac{1}{\eta} \left(\norm{\*u^{\xbf}_{k+1}-\xbf_k}+\norm{\*u^{\*z}_{k+1}-\*z_k}\right)
    \end{align}
\end{subequations}
for some $\eta>0$, we can accept $\xbf_{k+1}=\*u_{k+1}^{\xbf}$ and $\*z_{k+1}=\*u_{k+1}^{\*z} $. 
If any of \eqref{eq:uvconda} and \eqref{eq:uvcondb} is violated, we compute $(\*v^{\*z}_{k+1},\*v^{\xbf}_{k+1})$ by the standard alternating minimization, also known as the Block Coordinate Descent (BCD) algorithm, with a simple line-search strategy to safeguard the convergence: Let $\bar \alpha,\bar \beta$ be positive numbers in $(0,1)$
compute
\begin{align} \label{v1}
  &  \*v^{\*z}_{k+1}  =  \*z_k-\bar \alpha\left( \nabla_{\*z} f(\xbf_k,\*z_k)+\nabla Q_{\varepsilon_k}(\*z_k)\right),\\\label{v2}
 &   \*v_{k+1}^{\xbf}
     =\xbf_k-\bar \beta\left( \nabla_{\xbf} f(\xbf_k,\*v_{k+1}^{\*z})+\nabla R_{\varepsilon_k}(\xbf_k)\right).
\end{align}
Set $\xbf_{k+1}=\*v_{k+1}^{\xbf}$ and $\*z_{k+1}=\*v_{k+1}^{\*z}$, if for some $\delta\in (0,1)$, the following holds:
\begin{equation}
\label{v-condition-6}
\begin{split}
    \Phi_{\varepsilon}(\*v_{k+1}^{\xbf},\*v_{k+1}^{\*z})-\Phi_{\varepsilon}(\xbf_k,\*z_k)
\leq -\delta (\norm{\*v_{k+1}^{\xbf}-\xbf_k}^2+\norm{\*v_{k+1}^{\*z}-\*z_k}^2).
\end{split}
\end{equation}
Otherwise we reduce $(\bar \alpha,\bar \beta)\leftarrow \rho (\bar \alpha,\bar \beta)$ where $0<\rho<1$, and recompute $(\*v^{\xbf}_{k+1},\*v^{\*z}_{k+1})$ until the condition \eqref{v-condition-6} holds. 
We will show that \eqref{v-condition-6} will be achieved within finitely many steps later. Meanwhile, for the fixed $\varepsilon$, we will show that $\norm{\nabla\Phi_{\varepsilon}(\xbf_k,\zbf_k)}\to 0$ as $k\to\infty$. Overall, LAMA has a mechanism to decrease $\varepsilon$, and thus, what we described in this second stage can be viewed as an inner loop procedure for any particular $\varepsilon$.

The third stage of LAMA checks the value of $\norm{\nabla\Phi_{\varepsilon}(\xbf_k,\zbf_k)}$.  If it is small enough, then we repeat the second stage with a reduced smoothing factor $\gamma\varepsilon$ ($0<\gamma<1$). Otherwise, keep $\varepsilon$ unchanged. This yields a gradually decreasing sequence of $\varepsilon$, and we will obtain a subsequence of the iterates whose accumulation points are all Clarke stationary points (see Definition~\ref{def-clark-gen}) of \eqref{eq:OrgPhi}. LAMA \cite{b2} is summarized in Algorithm \ref{alg:LAMA}.

\begin{algorithm}[htb]
\caption{Linearized Alternating Minimization Algorithm \cite{b2}}
\textbf{Input:} Initializations: $\xbf_0$, $\*z_0$, $\delta$, $\eta$, $\rho$, $\gamma$, $\varepsilon_0$, $\sigma$, $\lambda$
\begin{algorithmic}[1]
    \FOR{$k=0,\,1,\,2,\,...$}
        \STATE $\*b_{k+1} =\*z_k-\alpha_k \nabla_{\*z} f(\xbf_k,\*z_k),$\label{alg:begin}
        \STATE $\*u^{\*z}_{k+1} = \*b_{k+1}-\hat{\alpha}_k\nabla Q_{\varepsilon_k}(\*b_{k+1})$
        \STATE$\*c_{k+1} =\xbf_k-\beta_k \nabla_{\xbf }f(\xbf_k,\*u^{\*z}_{k+1})$ \label{alg:c}
        \STATE $\*u^{\xbf}_{k+1} = \*c_{k+1}-\hat\beta_k\nabla R_{\varepsilon_k}(\*c_{k+1})$
        \IF{\eqref{eq:uvcond} holds}
            \STATE{$(\xbf_{k+1}, \*z_{k+1})\leftarrow (\*u^{\xbf}_{k+1}, \*u^{\*z}_{k+1})$}
        \ELSE
            \STATE{$\*v^{\*z}_{k+1}=\*z_k-\bar\alpha\left[ \nabla_{\*z}f(\xbf_k,\*z_k)+ \nabla Q_{\varepsilon_k}(\*z_{k})\right]$}\label{alg:v}
            \STATE{$\*v^{\xbf}_{k+1}=\xbf_k-\bar\beta \left[\nabla_{\xbf} f(\xbf_k,\*v^{\*z}_{k+1})+\nabla  R_{\varepsilon_k}(\xbf_k)\right]$}
            \STATE{\IfElse{\eqref{v-condition-6}}{$(\xbf_{k+1}, \*z_{k+1})\leftarrow (\*v^{\xbf}_{k+1}, \*v^{\*z}_{k+1})$}{$(\bar\beta, \bar\alpha) \leftarrow \rho (\bar\beta, \bar\alpha)$ and \textbf{go to} \ref{alg:v}}}
        \ENDIF
        \STATE{\IfElse{$\norm{\nabla \Phi_{\varepsilon_k}(\xbf_{k+1}, \*z_{k+1})}<\sigma\gamma\varepsilon_k$}{$\varepsilon_{k+1}=\gamma\varepsilon_k$}{$\varepsilon_{k+1}=\varepsilon_k$}}
    \ENDFOR
    \RETURN $\xbf_{k+1}$
\end{algorithmic}
\label{alg:LAMA}
\end{algorithm}

\section{Convergence Analysis of LAMA}
\label{sec:convergenceproof}

The main goal of this work is to prove an important convergence property of LAMA proposed in \cite{b2}, which is presented in this section. 
Such property brings advantages to generating an interpretable neural network.

Since the variational model \eqref{eq:OrgPhi} is a nonconvex and nonsmooth optimization problem, we need the following two definitions for more generalized derivatives.
\begin{definition} \label{def-clark-gen} (Clarke subdifferential). Suppose that $f:\mathbb{R}^n\times\mathbb{R}^m\rightarrow(-\infty,\infty]$ is locally Lipschitz. The Clarke subdifferential of $f$ at $(\xbf,\zbf)$ is defined as the set
    \begin{align*}
       \partial^c f(\xbf,\zbf)\defeq\{(\pbf, \qbf)\in\mathbb{R}^n\times \mathbb{R}^m~\vert~ \innerproduct{\pbf}{\rbf}+\innerproduct{\qbf}{\sbf} \le f^{\circ}(\xbf,\zbf;\rbf,\sbf),\ \forall\, (\rbf,\sbf) \in \mathbb{R}^n\times \mathbb{R}^m \}
    \end{align*}
where $f^{\circ}(\xbf, \zbf; \rbf, \sbf)$ stands for the following upper-limit:
\begin{equation}
    \label{eq:upper-lim}
    f^{\circ}(\xbf, \zbf; \rbf, \sbf) := \limsup_{(\cbf,\dbf) \to (\xbf,\zbf),t \downarrow 0} \frac{f(\cbf + t \rbf, \dbf + t \sbf) - f(\cbf, \dbf)}{t}.
\end{equation}
\end{definition}

\begin{definition}
    (Clarke stationary point) For a locally Lipschitz function $f$ defined as in Definition \ref{def-clark-gen}, a point $(\xbf, \zbf)\in \mathbb R^n\times \mathbb R^m$ is called a Clarke stationary point of $f$ if $\zerobf \in \partial^c f(\xbf, \zbf)$. 
\end{definition}

We remark that Clarke subdifferential is a generalization of standard subdifferential defined for convex but nonsmooth (nondifferentiable) functions and hence can be a set at a nondifferentiable point $(\xbf,\zbf)$. Meanwhile, Clarke stationary point is a generalization of classical stationary point for local Lipschitz functions, which can be both nondifferentiable and nonconvex. Thus, if $(\xbf^*,\zbf^*)$ is a local minimizer of $f$, then it is necessary that it is a Clarke stationary point. We intend to show that all subsequences of LAMA with specific conditions will converge to a Clarke stationary point later.

For a function $f(\xbf, \zbf)$, we denote the vector $(\nabla_\xbf f, \nabla_\zbf f)$ in short by $\nabla_{\xbf,\zbf} f $.  We start with a few lemmas before proving the main convergence result.

\begin{lemma}\label{lem:r_subdiff}
The Clarke subdifferential of $\|\*g^r(\*y)\|_{2,1}$ at $\ybf$ is as follows:
\begin{align}\label{eq:r_subdiff}
\partial^c \|g^r(\ybf)\|_{2,1} &= \bigg\{\sum_{i\in I_0}\nabla g^r_i(\ybf)^{\top}  \wbf_i + \sum_{i \in I_1}\nabla g^r_i(\ybf)^{\top}\frac{g^r_i(\ybf)}{\|g^r_i(\ybf)\|} \ \bigg\vert 
\ \wbf_i \in \mathbb{R}^{d_r}, \nonumber\\
&\qquad \quad \|\Pi(\wbf_i; \Ccal(\nabla g^r_i(\ybf)))\|\leq 1,\ \forall\, i \in I_0 \bigg\} ,  
\end{align}
where $\*g^r(\*y)\in\mathbb{R}^{m_r\times d_r}$, $I_0=\{i \in [m_r] \ | \ \|g^r_i(\ybf) \|= 0 \}$, $I_1=[m_r] \setminus I_0$, and $\Pi(\wbf;\Ccal(\Abf))$ is the projection of $\wbf$ onto $\Ccal(\Abf)$ which stands for the column space of $\Abf$. 
\end{lemma}

\begin{lemma}
    The gradient
    $\nabla r_{\varepsilon}(\ybf)$ in equation~\eqref{eq:smoothedR} is ($\sqrt{m} L_{g^r}+\frac{M^2}{\varepsilon}$)-Lipschitz continuous, 
    where $L_{g^r}$ is the Lipschitz constant of $\nabla g^r$ and $M=\sup_{y} \|\nabla g^r(y) \|_2$. Consequently, $\nabla_{\xbf,\zbf} \Phi_\varepsilon(\xbf, \zbf)$ is $L_{\varepsilon}$-Lipschitz continuous with $L_{\varepsilon}=O(\varepsilon^{-1})$.
\end{lemma}

The proofs of the two Lemmas above can be found in Lemmas 3.1 and 3.2 in \cite{b15} and thus omitted here.
The next Lemma discusses the convergence property for the sequence generated from lines 1 to 12 of the algorithm for minimizing  $\Phi_\varepsilon(\xbf, \zbf)$ with any fixed $\varepsilon$. 

\begin{lemma}
    Let $\Phi_\varepsilon(\xbf, \zbf)$ be defined in equation~\eqref{SM}. Let also $\varepsilon, \eta, a >0$, $\rho, \bar \alpha,\bar \beta \in (0,1)$, and $(\xbf_0, \zbf_0)$ be an arbitrary initial condition. 
     Suppose $\{\xbf_{k}, \zbf_{k}\}$ is the sequence generated by repeating Lines 1--10 of Algorithm 1  with  $\varepsilon_k = \varepsilon$ for all $k$. 
     Then $\norm{\nabla\Phi_{\varepsilon}(\xbf_k,\zbf_k)}\to 0$ as $k\to\infty$.
    \end{lemma}

\begin{proof}
Given $(\xbf_k, \zbf_k)$, in case that $(\*u^{\xbf}_{k+1},\*u^{\*z}_{k+1})$  satisfy the conditions  \eqref{eq:uvconda} and \eqref{eq:uvcondb}, we take
$(\xbf_{k+1},\zbf_{k+1})=(\*u^{\xbf}_{k+1},\*u^{\*z}_{k+1})$, and have
\begin{equation}\label{x=u}
\|\nabla_{\xbf,\zbf} \Phi_{\varepsilon}(\xbf_{k+1},\zbf_{k+1})\|^2\le \frac{2(\Phi_{\varepsilon}(\xbf_{k},\zbf_{k})-\Phi_{\varepsilon}(\xbf_{k+1},\zbf_{k+1}))}{\eta^3}.
\end{equation}

If $(\*u^{\xbf}_{k+1},\*u^{\*z}_{k+1})$ fails to satisfy any condition in \eqref{eq:uvcond}, the algorithm computes $(\*v^{\xbf}_{k+1},\*v^{\*z}_{k+1})$.

{Since $(\*v^{\xbf}_{k+1},\*v^{\*z}_{k+1})$ satisfies the condition \eqref{v-condition-6}, which is essentially the same condition as \eqref{eq:uvconda} but a different scaling factor on the right-hand side with $\delta$ instead of $\eta$. 
Next we show that the \eqref{x=u} holds if we set $(\xbf_{k+1},\zbf_{k+1}) = (\*v^{\xbf}_{k+1},\*v^{\*z}_{k+1})$.}  Suppose that the number of the line search required for the condition \eqref{v-condition-6}  is $\ell_k$. We have
\begin{subequations}
\label{eq:v-ls}
\begin{align}
\*v^{\xbf}_{k+1}
&=\xbf_k-\bar \alpha \rho^{\ell_k}\nabla_{\xbf}\Phi_{\varepsilon}(\xbf_k,\zbf_k), \label{v-12c1}\\ 
\*v^{\*z}_{k+1}
&=\zbf_k-\bar \beta \rho^{\ell_k}\nabla_{\zbf}\Phi_{\varepsilon}(\*v^{\xbf}_{k+1},\zbf_k).\label{v-12c2}
\end{align}    
\end{subequations}

By the $L_{\varepsilon}$-Lipschitz continuity of $\nabla_{\xbf,\zbf} \Phi_{\varepsilon}$ and equations above, we have 
\begin{align}\label{phi-to-v}
\Phi_{\varepsilon}(\*v^{\xbf}_{k+1},\*v^{\*z}_{k+1})
& \le \Phi_{\varepsilon}(\*v^{\xbf}_{k+1},\zbf_k) +\nabla_{\zbf}\Phi_{\varepsilon}(\*v^{\xbf}_{k+1},\zbf_k)\cdot
(\*v^{\*z}_{k+1}-\zbf_k)+\frac{L_{\varepsilon}}{2}\|\*v^{\*z}_{k+1}-\zbf_k\|^2\nonumber\\
& \le \Phi_{\varepsilon}(\xbf_k,\zbf_k)+\nabla_{\xbf}\Phi_{\varepsilon}(\xbf_k,\zbf_k)\cdot (\*v^{\xbf}_{k+1}-\xbf_k) +\frac{L_{\varepsilon}}{2}\|\*v^{\xbf}_{k+1}-\xbf_k\|^2 \\ 
& \qquad \quad \quad \quad +\nabla_{\zbf}\Phi_{\varepsilon}(\*v^{\xbf}_{k+1},\zbf_k)\cdot
(\*v^{\*z}_{k+1}-\zbf_k)+\frac{L_{\varepsilon}}{2}\|\*v^{\*z}_{k+1}-\zbf_k\|^2\nonumber\\
& \leq   \Phi_{\varepsilon}(\xbf_k,\zbf_k)+(\frac{L_{\varepsilon}}{2}-\frac 1{\bar \alpha\rho^{\ell_k}})\|\*v^{\xbf}_{k+1}-\xbf_k\|^2 +(\frac{L_{\varepsilon}}{2}-\frac{1}{\bar \beta\rho^{\ell_k}})\|\*v^{\*z}_{k+1}-\zbf_k\|^2 \nonumber. 
\end{align}
Hence, for any $k=1,2,\ldots$ the maximum line search steps $\ell_{max}$ required for equation~\eqref{v-condition-6} satisfies $\rho^{\ell_{max}}= (\delta+L_{\varepsilon}/2)^{-1}(\max\{\bar \alpha,\bar \beta\})^{-1}$. Note that $0 \le \ell_k \le  \ell_{max}$, hence,

\begin{equation} \label{rho-to-lmax}
  \min\{\bar \alpha,\bar \beta\}  \rho^{\ell_{k}}\geq 
  \min\{\bar \alpha,\bar \beta\} (\delta+L_{\varepsilon}/2)^{-1}(\max\{\bar \alpha,\bar \beta\})^{-1}.
    \end{equation}
Moreover, from \eqref{eq:v-ls}
we know that when the condition in \eqref{v-condition-6} is met, 
we have
\begin{align} \label{smoothphi-decay2}
\Phi_{\varepsilon}(\*v^{\xbf}_{k+1},\*v^{\*z}_{k+1})-\Phi_{\varepsilon}(\xbf_k,\zbf_k)
& \leq -\delta (\bar \alpha\rho^{\ell_k})^2\|\nabla_{\xbf}\Phi_{\varepsilon}(\xbf_k,\zbf_k)\| ^2 \\
& \qquad -\delta (\bar \beta \rho^{\ell_k})^2\|\nabla_{\zbf}\Phi_{\varepsilon} (\*v^{\xbf}_{k+1},\zbf_k)\|^2. \nonumber
\end{align}

Now we estimate the last term in \eqref{smoothphi-decay2}.
Using $L_{\varepsilon}$-Lipschitz of $\nabla_{\xbf,\zbf} \Phi_{\varepsilon}$ and the inequality 
$b^2\ge \mu a^2-\frac{\mu}{1-\mu}(a- b)^2$ for all $\mu\in(0,1)$,  we  get
\begin{align} \label{tem-1}
  & \|\nabla_{\zbf}\Phi_{\varepsilon}(\*v^{\xbf}_{k+1},\zbf_k)\|^2 - \mu\|\nabla_{\zbf}\Phi_{\varepsilon}(\xbf_{k},\zbf_k)\|^2 \nonumber\\
  &\geq -\frac{\mu}{1-\mu}\|\nabla_{\zbf}\Phi_{\varepsilon}(\*v^{\xbf}_{k+1},\zbf_k)-\nabla_{\zbf}\Phi_{\varepsilon}(\xbf_{k},\zbf_k)\|^2 \nonumber \\   
 & \geq 
 -\frac{\mu}{1-\mu}
  L_{\varepsilon}^2\|\*v^{\xbf}_{k+1}-\xbf_k\|^2 
\nonumber\\
&=-\frac{\mu}{1-\mu}
  L_{\varepsilon}^2 (\bar \alpha \rho^{\ell_k})^2\|\nabla_{\xbf}\Phi_{\varepsilon}(\xbf_{k},\zbf_k)\|^2.
\end{align}

Inserting \eqref{tem-1} into \eqref{smoothphi-decay2}, we have 
\begin{align} \label{line-up-3}
&\Phi_{\varepsilon}(\*v^{\xbf}_{k+1},\*v^{\*z}_{k+1})-\Phi_{\varepsilon}(\xbf_k,\zbf_k)\nonumber\le -\delta\mu(\bar \beta \rho^{\ell_k})^2\|\nabla_{\zbf}\Phi_{\varepsilon}(\xbf_{k},\zbf_k)\|^2\nonumber\\
  & -\delta
  (\bar \alpha \rho^{\ell_k})^2\big(1-\frac{\mu}{1-\mu}(\bar \beta \rho^{\ell_k})^2 L_{\varepsilon}^2\big )\|\nabla_{\xbf}\Phi_{\varepsilon}(\xbf_{k},\zbf_k)\|^2. 
\end{align}
Then from \eqref{rho-to-lmax} and \eqref{line-up-3}, there are a sufficiently small $\mu$ and a constant $C_1>0$, depending only on $\rho$, $\delta$, $\bar \alpha$, $\bar \beta$, such that 
\begin{equation} \label{x=v}
 \|\nabla_{\xbf,\zbf}\Phi_{\varepsilon}(\xbf_{k},\zbf_k)\|^2   \leq C_1L_{\varepsilon}^2\big(\Phi_{\varepsilon}(\xbf_k,\zbf_k)-\Phi_{\varepsilon}(\*v^{\xbf}_{k+1},\*v^{\*z}_{k+1})\big).
\end{equation}
{This shows that $(\xbf_{k+1},\zbf_{k+1})$ satisfies the same type of estimates as  \eqref{x=u}, with a different and unimportant scaling factor on the right-hand side if $(\xbf_{k+1},\zbf_{k+1})=(\*v^{\xbf}_{k+1},\*v^{\*z}_{k+1})$. Note that the subscripts on the left-hand side are $k$, not $k+1$ as in \eqref{x=u}, but this will not affect the conclusion of the telescoping sum later.}

{Hence, from \eqref{eq:uvconda}, \eqref{x=u}, \eqref{v-condition-6} and \eqref{x=v},} either $(\xbf_{k+1},\zbf_{k+1})=(\*u^{\xbf}_{k+1},\*u^{\*z}_{k+1})$ or $(\xbf_{k+1},\zbf_{k+1})=(\*v^{\xbf}_{k+1},\*v^{\*z}_{k+1})$, we have
    \begin{align}
    \label{eq:xconda}
    \Phi_{\varepsilon}({\xbf}_{k+1},{\*z}_{k+1})-\Phi_{\varepsilon}(\xbf_k,\*z_k) &\leq -C_2\left(\norm{{\xbf}_{k+1}-\xbf_k}^2+\norm{{\*z}_{k+1}-\*z_k}^2\right),\\
    \label{eq:xcondb}
       \norm{\nabla \Phi_{\varepsilon}(\xbf_{k},\*z_{k})}^2&\leq C_3\big(\Phi_{\varepsilon}(\xbf_k,\zbf_k)-\Phi_{\varepsilon}(\xbf_{k+1},\*z_{k+1})\big),
    \end{align}
    where $C_2=\min\{\eta, \delta \}$, $C_3= \max\{2/\eta^3, C_1L_{\varepsilon}^2\}$. 
    
    From equation~\eqref{eq:xconda}, It is easy to conclude that there is a $\Phi_{\varepsilon}^*$, such that {$\Phi_{\varepsilon}(\xbf_{k},\*z_{k})\downarrow \Phi_{\varepsilon}^*$ as $k\to\infty$.} Then adding up both sides of equation~\eqref{eq:xcondb} w.r.t. $k$, it yields that for any $K$, $$\sum_{k=0}^K \norm{\nabla \Phi_{\varepsilon}(\xbf_{k},\*z_{k})}^2\leq C_3(\Phi_{\varepsilon}(\xbf_0,\*z_0)-\Phi_{\varepsilon}^*) < \infty.$$
This leads to the immediate conclusion of the lemma.
\end{proof}

Now, we are ready to show the main convergent results.
\begin{theorem} \label{thm}
    Let $\{(\xbf_k, \zbf_k)\}$ be the sequence generated by the algorithm with arbitrary initial condition $(\xbf_0, \zbf_0)$, arbitrary $\varepsilon_0>0$ and $\varepsilon_{tol}=0$. Let $\{(\xbf_{k_l+1}, \zbf_{k_l+1})\}$ be the subsequence, where the reduction criterion in the algorithm is met for $k=k_l$ and $l=1,2, ...$. Then  $\{(\xbf_{k_l+1}, \zbf_{k_l+1})\}$ has at least one accumulation point, and each accumulation point is a Clarke stationary point of $\Phi$. 
\end{theorem}

\begin{proof}
   For notation simplicity, in this proof, we drop the subscript $r$ in \eqref{eq:smoothedR}, which doesn't affect the proof.

   By the definition of $r_{\varepsilon, i}(\*y)$ in \eqref{eq:smoothedR}, we have
\begin{equation}
    r_{\varepsilon,i}(\*y) + \frac{\varepsilon}{2} := \begin {cases}
\| \gbf_i(\*y) \|,  & \mbox{if} \ i\in I_1, \\
\frac{\| \*g_i(\*y) \|^2}{2\varepsilon} 
+ \frac{\varepsilon}{2},
\quad & \mbox{if} \ i\in I_0.
\end{cases}
\end{equation}   
Note that $\frac{\| \*g_i(\*y) \|^2}{2\varepsilon} 
+ \frac{\varepsilon}{2}$ as a function of $\varepsilon$ is non-decreasing. It is clear that for any $i$ either in $I_0$ or in $I_1$,  $r_{\varepsilon,i}(\*y) + \frac{\varepsilon}{2}$ is non-decreasing in $\varepsilon$. Hence, by the definitions  of $r_{\varepsilon,i}(\*y)$ in \eqref{eq:smoothedR} and $\Phi_{\varepsilon}$ in \eqref{SM}, we have the first inequality below for any $k\ge 0$,
\begin{equation}\label{eq:phi_decay}
    \Phi_{\varepsilon_{k+1}}(\xbf_{k+1}, \*z_{k+1}) + \frac{m \varepsilon_{k+1}}{2} \le \Phi_{\varepsilon_{k}}(\xbf_{k+1}, \*z_{k+1}) + \frac{m \varepsilon_{k}}{2} \le \Phi_{\varepsilon_{k}}(\xbf_{k}, \*z_{k}) + \frac{m \varepsilon_{k}}{2},
\end{equation}
where the second inequality comes from \eqref{eq:xconda}. Combine \eqref{eq:phi_decay} with the fact that $\Phi(\xbf, \*z) \le \Phi_{\varepsilon}(\xbf, \*z) + \frac{m \varepsilon}{2}$ for any $\varepsilon>0$ and all $(\xbf, \*z)$, we get
\begin{equation*}
\Phi(\xbf_k, \*z_k) \le \Phi_{\varepsilon_k}(\xbf_k, \*z_k) + \frac{m\varepsilon_k}{2} \le \cdots \le \Phi_{\varepsilon_0}(\xbf_0, \*z_0) + \frac{m\varepsilon_0}{2} < \infty.
\end{equation*}
Since $\Phi$ is coercive, the sequence $\{(\xbf_k, \zbf_k)\}$ must be bounded, so does 
$\{(\xbf_{k_l+1}, \zbf_{k_l+1})\}$. Hence, $\{(\xbf_{k_l+1}, \zbf_{k_l+1})\}$ has at least one accumulation point. 
Let $\{(\xbf_{k_{l_j}+1}, \zbf_{k_{l_j}+1})\}$ be one of the convergent subsequence of $\{(\xbf_{k_l+1}, \zbf_{k_l+1})\}$, and denote its limit point by $(\bar{\xbf}, \bar{\*z})$.
Note that $(\xbf_{k_l+1}, \zbf_{k_l+1})$ satisfies the reduction criterion in Line 13 of Algorithm 1, hence for this subsequence, we have that as $j \to \infty$,
\begin{equation} \label{gradphi}
    \| \nabla_{\xbf,\zbf} \Phi_{\varepsilon_{k_{l_j}}} (\xbf_{k_{l_j}+1}, \zbf_{k_{l_j}+1}) \| \le \sigma \gamma \varepsilon_{k_{l_j}+1} = \sigma \varepsilon_0 \gamma^{{l_j}+1} \to 0.
\end{equation}

Next, we are going to show that
\begin{equation} \label{lim grad smoothPhi}
    \lim_{j \rightarrow \infty}\nabla_{\xbf,\zbf} \Phi_{\varepsilon_{k_{l_j}}} (\xbf_{k_{l_j}+1}, \zbf_{k_{l_j}+1})\in \partial^{c} \Phi (\bar{\xbf},\bar{\*z}).
\end{equation}

The Clarke subdifferential of $\Phi$ at $(\bar{\xbf}, \bar{\*z})$ is given by $\partial^{c} \Phi(\bar{\xbf}, \bar{\*z}) = \partial^{c} R(\bar{\xbf}) + \partial^{c} Q(\bar{\*z}) + \nabla_{\xbf,\zbf} f(\bar{\xbf}, \bar{\*z})$. 
By Lemma 1, we have
\begin{align}
   &\partial^{c} \Phi(\bar{\xbf}, \bar{\*z})
   =\{ \sum_{i \in I^R_0} \nabla \*g^R_i(\bar{\xbf})^{\top} \wbf_{1,i} + \sum_{ i \in I^R_1} \nabla \*g^R_i(\bar{\xbf})^{\top} \frac{\*g^R_i(\bar{\xbf})}{\| \*g^R_i(\bar{\xbf}) \|} +
\sum_{i \in I^Q_0} \nabla \*g^Q_i(\bar{\*z})^{\top} \wbf_{2,i}  \nonumber \\
&+\sum_{ i \in I^Q_1} \nabla \*g^Q_i(\bar{\*z})^{\top} \frac{\*g^Q_i(\bar{\*z})}{\| \*g^Q_i(\bar{\*z}) \|} 
 + \nabla_{\xbf,\zbf} f(\bar{\xbf}, \bar{\*z}) \ \bigg\vert \ (\wbf_{1,i}, \wbf_{2,i}) \in R^{d_R} \times R^{d_Q} \ \mbox{satisfying }  \nonumber\\ 
 &\| \Pi(\wbf_{1,i}; \Ccal(\nabla \*g^R_i(\bar{\xbf}) ))\le 1,\ \forall\, i\in I^R_0, \ \mbox{and} \
 \| \Pi(\wbf_{2,i}; \Ccal(\nabla \*g^Q_i(\bar{\*z}) ))\le 1,\ \forall\, i\in I^Q_0\},  \label{eq:d_phi_bar}
\end{align}
where $I^R_0 = \{i\in[m]\ \vert \ \|\*g^R_i(\bar{\xbf})\| = 0 \}$,  $I^R_1 = [m] \setminus I^R_0$, and 
$I^Q_0 = \{i\in[m]\ \vert \ \|\*g^Q_i(\bar{\*z})\| = 0 \}$,  $I^Q_1 = [m] \setminus I^Q_0$. The first two terms in the summation is the $\partial^{c} R(\bar{\xbf})$, the second two terms gives $  \partial^{c} Q(\bar{\*z})$. 

From \eqref{eq:gradsmoothedR}, we have \begin{equation}\label{eq:d_R_eps}
\nabla R_{\varepsilon_{k_{l_j}}}(\xbf_{k_{l_j}+1}) = \sum_{i \in I^R_0} \nabla \*g^R_i(\xbf_{k_{l_j}+1})^{\top} \frac{\*g^R_i(\xbf_{k_{l_j}+1})}{\varepsilon_{k_{l_j}} }+ \sum_{ i \in I^R_1} \nabla \*g^R_i(\xbf_{k_{l_j}+1})^{\top} \frac{\*g^R_i(\xbf_{k_{l_j}+1})}{\|\*g^R_i(\xbf_{k_{l_j}+1}) \|}.
\end{equation}

Since $\*g^R_i$, $\nabla \*g^R_i$ are continuous, 
$(\xbf_{k_{l_j}+1}, \*z_{k_{l_j}+1})\rightarrow (\bar{\xbf}, \bar{\*z}) $, as $j \rightarrow \infty$, and $\|\frac{\*g^R_i(\xbf_{k_{l_j}+1})}{\varepsilon_{k_{l_j}}}\|\le 1, $ for $i\in I^R_0$, 
$\frac{\*g^R_i(\xbf_{k_{l_j}+1})}{\|\*g^R_i(\xbf_{k_{l_j}+1}) \|}\le 1$, for $i\in I^R_1$,
from \eqref{eq:r_subdiff} one can see that 
\begin{equation}\label{lim grad smoothR}
    \lim_{j \rightarrow \infty}\nabla_{\xbf} R_{\varepsilon_{k_{l_j}}}(\xbf_{k_{l_j}+1})\in \partial^{c} R (\bar{\xbf}).
\end{equation}
Similarly, 
\begin{equation}\label{lim grad smoothQ}
    \lim_{j \rightarrow \infty}\nabla_{\zbf} Q_{\varepsilon_{k_{l_j}}}(\xbf_{k_{l_j}+1})\in \partial^{c} Q (\bar{\*z}).
\end{equation}
Combining \eqref{lim grad smoothR} and \eqref{lim grad smoothQ}, and comparing with \eqref{eq:d_phi_bar}, we proved \eqref{lim grad smoothPhi}.

Note that $\partial^{c} \Phi(\bar{\xbf},\bar{\*z})$ is closed, then from  \eqref{gradphi} and \eqref{lim grad smoothPhi} we conclude that $0 \in \partial^{c} \Phi(\bar{\xbf},\bar{\*z})$. Hence $(\bar{\xbf},\bar{\*z})$ is a  Clarke stationary point of $\Phi (\xbf, \*z)$.
\end{proof}

\section{LAMA-Net and iLAMA-Net}
\label{sec:lama-net}
As shown in \cite{b2}, the architecture of LAMA-Net is unrolled by the LAMA, where every layer of LAMA-Net executes one iteration of LAMA in Algorithm \ref{alg:LAMA}. The only difference is that LAMA-Net merely contains a small number of phases (such as 15 as in the training process) for the best balance between solution quality and speed.

In this work, we also introduce potential approaches to further improve the performance of LAMA-Net by integrating a properly designed Init-Net. This is because the original problem \eqref{eq:OrgPhi} is nonconvex and nonsmooth, and therefore, it is important to choose a good initial for LAMA-Net. Since we will focus on the application of LAMA-Net on SVCT, we provide an example of such Init-Net designed for SVCT below. Then, we will show through experiments that the combination of this network and LAMA-Net, which we call iLAMA-Net, can further improve the solution quality. Because LAMA-Net is a general method applicable to dual-domain reconstruction, one can build an iLAMA-Net by integrating a different initialization network for a specific application with a similar principle.

\subsection{Initialization Network}
\label{subsect:init}
In this subsection, we show the Init-Net to improve the performance of LAMA-Net for SVCT reconstruction. 
We let $\sbf$ be the full-view sinogram data of an image $\xbf$, and $\sbf_0$ be the sparse-view data acquired during the scan. Suppose $\Abf$ is the discrete Radon transform and $\Pbf_{0}$ is the binary matrix representing the selection of the sparse-view data $\sbf_0$, then we have:
\begin{equation}
\label{eq:Ax-s}
    \sbf = \Abf \xbf \quad \mbox{and} \quad \Pbf_{0} \sbf = \Pbf_{0} \Abf \xbf = \sbf_0.
\end{equation}

Recall that the core technique here is to build a deep neural network $\bar{\Psi}_{\omega}$ with parameters denoted by $\omega$, such that it can map any sparse-view data $\sbf_0$ to $\bar{\Psi}_{\omega}(\sbf_0)$, the FBP of a pseudo full-view sinogram, such that $\bar{\Psi}_{\omega}(\sbf_0)$ is close to the corresponding ground truth image $\xbf$. 
The structure of $\bar{\Psi}_{\omega}$ and the procedure to learn the parameter $\omega$ will be explained below.

Suppose the sparse-view scan only acquires $1/p$ of the full-view data, where $p \in \Nbb$ is the downsampling rate, then the full-view sinogram $\sbf$ can be rearranged into $\sbf = [\sbf_0,\sbf_1,\dots,\sbf_{p-1}]$. Here the sparse-view sinogram data $\sbf_0=[\sbf_0^{1},\sbf_0^{2},\dots,\sbf_0^{V}]$ contains a total of $V$ views, and $\sbf_0^j$ is the measurement vector obtained by the detector at angle $\vartheta_{jp}$ for partitions indexed by $j=1,\dots,V$. See Figure \ref{fig:ct} for demonstration using $p=4$,
where $\Delta \vartheta = (\vartheta_1-\vartheta_0) = \dots = (\vartheta_p - \vartheta_{p-1})$ is the angle between two adjacent angles in a full-view sinogram. 
\begin{figure}
    \centering
    \includegraphics[width=.5\textwidth]{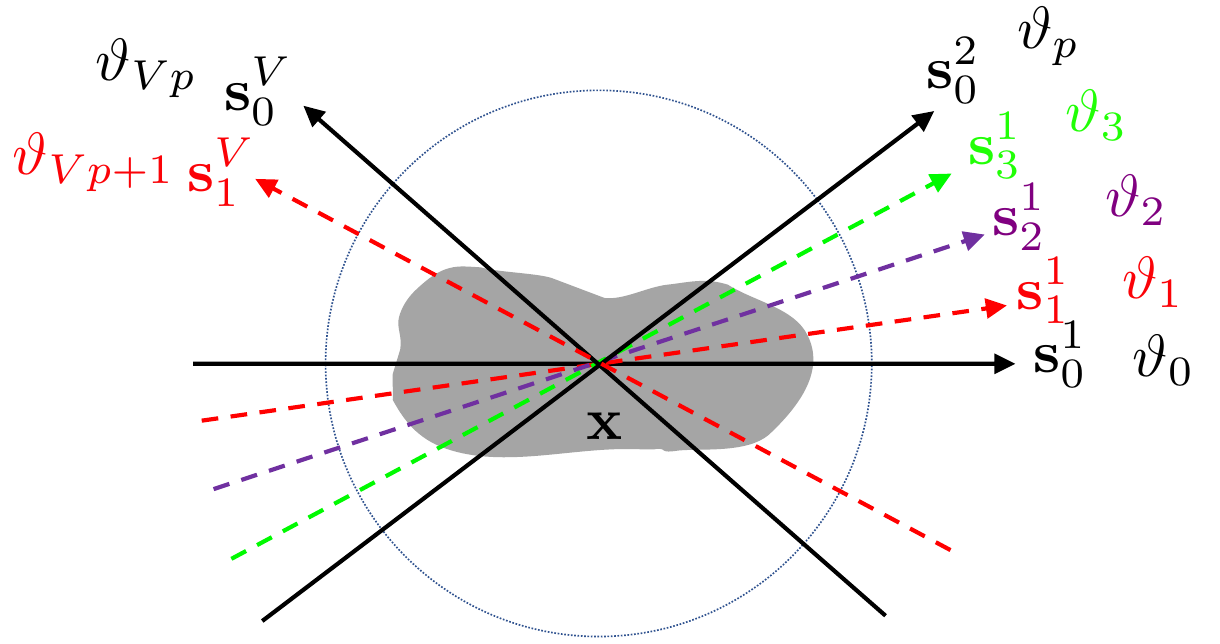}
    \caption{Demonstration of recursive relations between measurement data ${\sbf}$ and view angles $\vartheta$. The gray area indicates the subject ${\xbf}$. Solid black arrows show the angles acquired in the sparse-view CT scan, and dash red/purple/green arrows show the skipped views. Here ${\sbf}=[{\sbf}_{0},\dots,{\sbf}_{p-1}]$ is the full-view data for $p=4$, ${\sbf}_0=[{\sbf}_0^1,\dots,{\sbf}_0^{V}]$ is the sparse-view data containing the measurement vectors at a total of $V$ views, and ${\sbf}_0^j$ is the measurement vector obtained by the detector at angle $\vartheta_{jp}$ for partitions $j=1,\dots,V$.}
    \label{fig:ct}
\end{figure}
Let $\Pbf_{i}$ be the selection matrix corresponding to $\sbf_i$, i.e., $\Pbf_{i} \sbf = \sbf_i$, for $i=1,\dots,p$. Then we have the following relation between $\sbf_0$ and $\sbf_1$:
\begin{equation}
\label{eq:s1-s0}
\sbf_1 = \Pbf_1 \sbf = \Pbf_1 \Abf \xbf \approx \Pbf_1 \Abf \bar{\Psi}_{\omega}(\sbf_0) = \Psi_{\omega}(\sbf_0), 
\end{equation}
where ${\Psi}_{\omega} := \Pbf_1 \Abf \bar{\Psi}_{\omega}$. Then ${\Psi}_{\omega}$ is essentially the mapping from $\sbf_0$ to $\sbf_1$, which is the building block of the initialization network $\bar{\Psi}_\omega$.
To this end, we will use a CNN $\Psi_{\omega}$ to directly approximate the mapping from $\sbf_0$ to $\sbf_1$ and use it to construct $\bar{\Psi}_{\omega}$ later.

Note that if the CT scanner undersamples the sinogram starting from the angle $\vartheta_1$ rather than $\vartheta_0$ in Figure \ref{fig:ct}, then $\sbf_1$ is obtained instead.
This is equivalent to using the same scan pattern in Figure \ref{fig:ct} but with the subject $\xbf$ rotated \textit{clockwise} for angle $\Delta \vartheta$ in the scan.
As such, we should have $\sbf_2 \approx \Psi_{\omega}(\sbf_1)$ for the same reason of \eqref{eq:s1-s0} because $\sbf_2$ to $\sbf_1$ is as $\sbf_1$ to $\sbf_0$ now. 
In other words, $\Psi_{\omega}$ approximates the mapping from a sparse-view sinogram data to the one rotated counterclockwise for $\Delta \vartheta$.
Hence, we expect to have the following recurrent relation hold approximately:
\begin{equation}
    \label{eq:chain-relation}
    \sbf_{0} \stackrel{\Psi_{\omega}}{\longmapsto} \sbf_1 \stackrel{\Psi_{\omega}}{\longmapsto} \cdots \stackrel{\Psi_{\omega}}{\longmapsto} \sbf_{p-1} \stackrel{\Psi_{\omega}}{\longmapsto} \sbf_p:=[\sbf_0^2,\cdots,\sbf_0^V,\sbf_0^1]
\end{equation}
In practice, we use a set of $N$ ground truth full-view sinogram data $\{\sbf^{(n)}: n\in [N]\}$ and train a CNN $\Psi_{\omega}$ with parameter $\omega$ by minimizing
\begin{equation}
    \label{eq:train-Psi}
    \min_{\omega}\ (1/pN)\cdot \textstyle\sum_{n=1}^{N}\sum_{i=1}^{p} \| \Psi_{\omega} (\sbf_{i-1}^{(n)}) - \sbf_{i}^{(n)}\|^2.
\end{equation}
After training, we obtain $\Psi_\omega$, which can map a sparse-view data $\sbf_0$ to a pseudo full-view sinogram $\zbf_\omega(\sbf_0)$, and define the initialization $\bar{\Psi}_\omega$ as follows: 
\begin{align}
    \label{eq:init-net}
    \zbf_\omega(\sbf_0) := [\sbf_0,\Psi_{\omega}(\sbf_0),\dots,\Psi_{\omega}^{(p-1)}(\sbf_0)] \quad \mbox{and} \quad
    \bar{\Psi}_{\omega}(\sbf_0) := \text{FBP}(\zbf_\omega(\sbf_0))
\end{align}
where $\Psi_{\omega}^{(i)}$ stands for the composition of $\Psi_{\omega}$ for $i$ times.
Then, from a sparse-view sinogram data $\sbf_0$, the sinogram-image pair $(\zbf_\omega(\sbf_0), \bar{\Psi}_{\omega}(\sbf_0))$ is obtained and can serve as the input of LAMA-Net. We call this mapping $\sbf_0 \mapsto (\zbf_\omega(\sbf_0), \bar{\Psi}_{\omega}(\sbf_0))$ the Init-Net.
This recursive approach in Init-Net is highly effective since it leverages the geometric structure of sparse-view sinograms, where neighboring projections exhibit stronger correlations, and thus produces good initials for LAMA-Net. More experimental details can be found in Table \ref{tab:init}.

\subsection{Architecture of LAMA-Net and iLAMA-Net}
\label{subsect:lama-net}

As mentioned before, the architecture of the LAMA-Net follows LAMA exactly in the sense that each phase of LAMA-Net is just a realization of an iteration of LAMA shown in Algorithm \ref{alg:LAMA}. An illustration of LAMA-Net is given in Figure \ref{fig:diagram}. 
With the Init-Net introduced in Section \ref{subsect:init} attached to the front of LAMA-Net, it becomes the variant iLAMA-Net.
We remark that LAMA-Net/iLAMA-Net inherits all the convergence properties of LAMA, making the solution more stable than manually designed unrolling networks. 
Moreover, LAMA-Net/iLAMA-Net effectively leverages complementary information through the inter-domain connections and is memory-efficient as the network parameter $\Theta$ is shared across all phases. The networks \( g_R \) and \( g_Q \) each comprise four convolutional layers with 32 output channels, stride 1, and activation \( a(\cdot) \).  The kernel sizes for \( g_R \) and \(g_Q\) are \( 3 \times 3 \) and \( 3 \times 15 \) with padding 1 and \( 1 \times 7 \) respectively, preserving spatial resolution. The Init-Net consists of 3 sequential blocks with skip connection, each containing four convolutional layers with kernels of size $3 \times 15$, stride 1, and padding $1 \times 7$, where all layers use ReLU activation. Rectangular kernels outperform their square counterparts in this context, likely because they capture more contextual information for each view, potentially aligning with the geometry of the CT data acquisition process.

\begin{figure}
     \centering
     \subfloat[\label{fig:diagram-main}]{%
     \includegraphics[width=\textwidth]{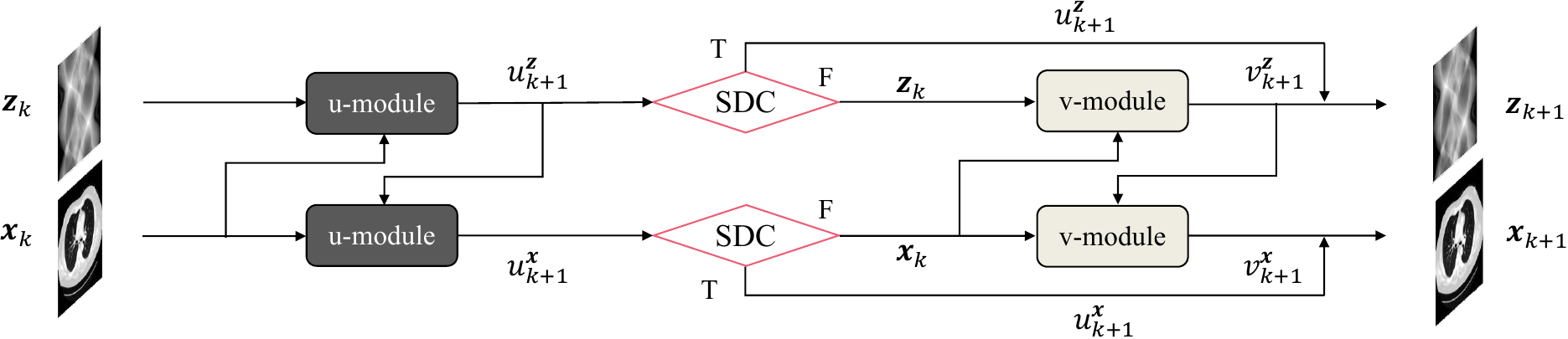}
         } 
     \hfill
     \subfloat[\label{fig:diagram-net}]{
         \centering
         \includegraphics[width=0.8\textwidth]{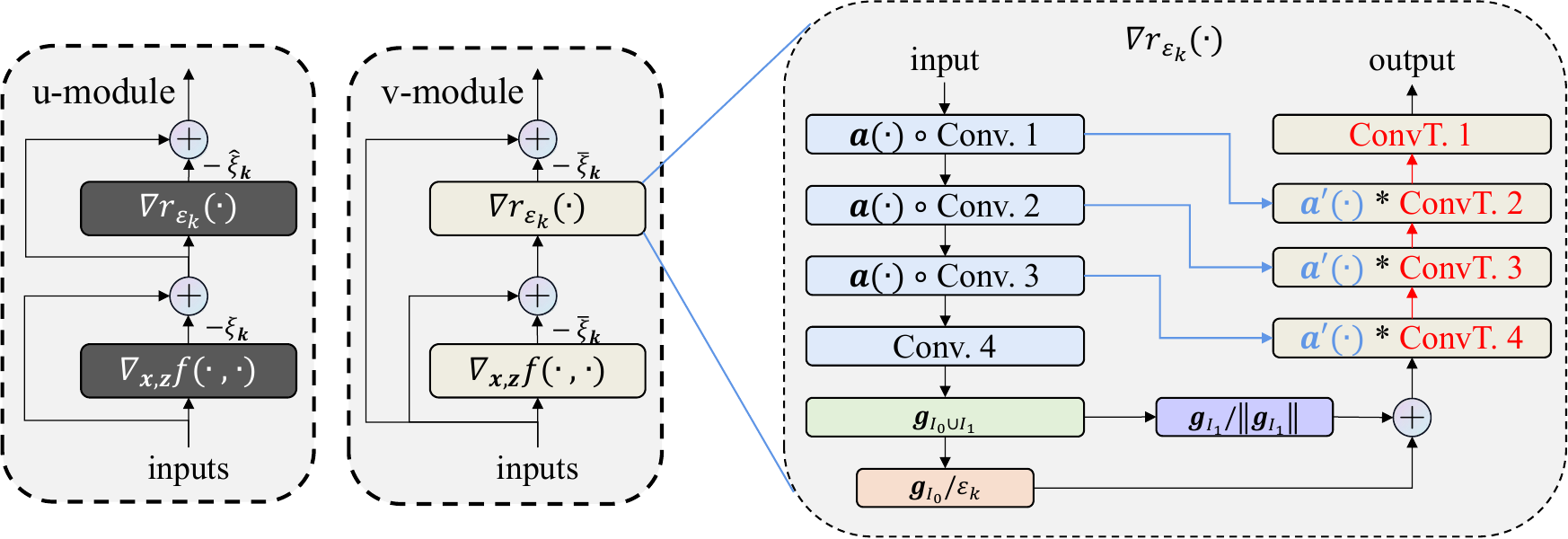}
         }
     \caption{Structure of one LAMA phase. (a) Schematic illustration of one phase in LAMA, where u-module employs residual learning architecture and v-module safeguards convergence. The $k$th phase has $(\xbf_k, \*z_k)$ as input and $(\xbf_{k+1}, \*z_{k+1})$ as output. LAMA is constructed by $K$ such phases. (b) Structures of the u-module, v-module, and the unrolling of the gradient over smoothed regularization operator $\nabla r_{\varepsilon_k}$ defined in equation~\eqref{eq:gradsmoothedR}. Note that $\xi$ denotes either $\alpha$ or $\beta$. $f$ is the data fidelity and data consistency function defined in equation~\eqref{eq:datafidelity} while $\nabla_{\xbf,\*z}f$ denote the gradient of $f$ w.r.t. either $\xbf$ or $\*z$. Conv. $i$ and ConvT. $i$ denote the convolutional layer and convolution transpose with weights of the $i^{th}$ corresponding instance, respectively. $a(\cdot)$ is the smoothed activation function, and $a'(\cdot)$ is the derivative of $a(\cdot)$. The unrolling of $\nabla r_{\varepsilon_k}$ into a neural network that exactly follows the mathematical formulation.}
     \label{fig:diagram}
\end{figure}


\section{Numerical Tests}

\label{sec:results}
We demonstrate the performance of the proposed method LAMA on dual-domain SVCT reconstruction. In this case, we use two well-recognized benchmark CT datasets: (i) the 2016 NIH-AAPM-Mayo (AAPM-Mayo) Clinic Low-Dose CT Grand Challenge dataset and (ii) the National Biomedical Imaging Archive (NBIA) dataset. More details of data and noise generation can be found below.

\subsection{Experiment Setup}

Among recent DL-based methods for comparison purposes, DDNet \cite{b8}, and LDA \cite{b15}  are single-domain methods, while DuDoTrans \cite{b42} and Learn++ \cite{b14} are dual-domain methods. Notably, except for DuDoTrans, dual-domain methods have proven very effective for SVCT reconstruction. This indicates that leveraging reconstructions in image and measurement domains can improve reconstruction quality. Our approach, LAMA-Net, is a mathematically unrolled method to building dual-domain reconstruction networks with convergence guarantee in theory and improved performance in practice. While the convergence has been demonstrated in Section \ref{sec:convergenceproof}, the empirical improvements will be shown below, such as Table \ref{n_param} and Figure \ref{fig:diagram}. 
These significant improvements are due to the better interpretability and optimization algorithmic design of the optimization method by LAMA-Net.

To implement the Init-Net in Section \ref{subsect:init}, we employ the simple convolutional neural network architecture described in Section \ref{subsect:lama-net} and the training data to learn its parameters. To implement LAMA-Net, we strictly follow the architecture in Section \ref{subsect:lama-net} and run the algorithm in Algorithm \ref{alg:LAMA}. We use the standard network training process on a 3-phase LAMA-Net to obtain its parameters, denoted by $\Theta_3$. Then, we increase the phase number of LAMA-Net to 5 and $\Theta=\Theta_3$ to start this new round of training to obtain an improved $\Theta = \Theta_5$. This process is repeated multiple times: we increase the phase number of LAMA-Net by two and use $\Theta$ obtained from the previous round for an updated $\Theta$. We often proceed with this until LAMA-Net reaches 15 phases because, after that, we do not observe much noticeable improvement in LAMA-Net. This recursive approach for training unrolling-type networks is particularly efficient, as shown in \cite{b11}. This training process is consistently used in all of our experiments.

In the experiments throughout this work, we randomly select 500 and 200 image-sinogram pairs from AAPM-Mayo and NBIA, respectively. The dataset is prepared by first applying the downsampling mask defined in \eqref{eq:datafidelity}, to the full-view ground truth sinograms \(\*s\). This process generates the zero-filled sparse-view sinograms \(\*s_0\). Subsequently, the corresponding FBP reconstruction iamge \(\text{FBP}(\*s_0)\) is computed from the sparse-view sinogram \(\*s_0\). For the ground truth image \(\hat{\*x}\), the FBP algorithm is applied directly to the full-view ground truth sinogram \(\*s\).
We randomly split the data in each dataset into two parts: 80\% for training and 20\% for testing. We call $N$ the number of training data pairs.
We evaluate algorithms using the peak signal-to-noise ratio (PSNR), structural similarity (SSIM) \cite{b43}, and the number of parameters in all compared networks. 
The sinogram has 512 detector elements and 1024 evenly distributed full projection views. 

In this SVCT experiment, the 1024 sinogram views are downsampled into 64 ($6.25\%$) or 128 ($12.5\%$) views to reduce the harmful X-ray dose to patients.
The full image size is $256\times256$, and we simulate projections and back-projections in fan-beam geometry using distance-driven algorithms \cite{b44} implemented in a PyTorch-based library CTLIB \cite{b45}.

Given $N$ training data pairs $\{( \*s_0^{(i)}, \hat{\xbf}^{(i)})\}_{i=1}^N$, where $\*s_0^{(i)}$ is sparse-view CT data and $\hat{\xbf}^{(i)}$ is the corresponding ground truth image, the loss function for learning the regularization parameter $\Theta$ is defined as:
\begin{equation}
\begin{split}
    \mathcal{L}(\Theta) = \frac 1 N\sum_{i=1}^N \norm{\xbf_{K}^{(i)}-\hat{\xbf}^{(i)}}^2 + \norm{\*z_{K}^{(i)}-\*A\hat{\xbf}^{(i)}}^2 + \mu\left(1-\text{SSIM}\big({\xbf}_{K}^{(i)}, \hat{\xbf}^{(i)}\big)\right),
\end{split}
\label{Loss_func}
\end{equation}
where $\mu$ is the weight for SSIM loss set as $0.01$ for all experiments, and $(\xbf_{K}^{(i)}, \*z_{K}^{(i)}) \defeq \text{LAMA}(\xbf_0:=\text{FBP}(\*s_0^{(i)}),\*z_0:=\*s_0^{(i)})$ by applying LAMA for a prescribed $K$ iterations ($K$ set to 15 in our experiments).

The above setting applies to training LAMA-Net without any initialization networks. To train the initialization network \(\bar{\Psi}_\omega\) (or equivalently $\Psi_\omega$), we provide the sparse-view sinograms \(\*s_0\) as inputs and the true full-view sinogram $\*s$ as ground truth. Then we train Init-Net in a supervised manner for 100 epochs using ADAM \cite{b46} with learning rate $10^{-4}$ to minimize the loss function defined in Equation \eqref{eq:train-Psi}. Once trained, the initialization network produces the refined sinogram-image pairs \(\{(\zbf_\omega(\sbf_0^{(i)}), \bar{\Psi}_{\omega}(\sbf_0^{(i)}))\}_{i=1}^N\), which is then provided as input to LAMA-Net. Then LAMA-Net with initialization (i.e. iLAMA-Net) is trained in a supervised manner with \((\xbf_{K}^{(i)}, \*z_{K}^{(i)}) \defeq \text{LAMA}(\xbf_0 := \bar\Psi_\omega(\sbf_0^{(i)}), \*z_0 := \zbf_\omega(\sbf_0^{(i)}))\). Note that the initialization network and LAMA-Net are trained based on the same data set but are trained independently to ensure no information leakage.

We use the ADAM optimizer \cite{b46} with learning rates of 1e-4 and 6e-5 for the image and sinogram networks, respectively, and train them with the recursive approach. All the other parameters are set to the default of ADAM. The training starts with three phases for 300 epochs, then adding two phases for 200 epochs each time until the number of phases reaches 15. The algorithm is implemented in PyTorch-Python version 1.10. Our experiments were run on a Linux server with an 80G NVIDIA A100 Tensor Core GPU. 

We show experimental results on two different testing setups using the two datasets above: The first test compares reconstruction quality directly to several state-of-the-art deep reconstruction networks, and the second test verifies LAMA's excellent stability connected to its convergence property.

\begin{table}[htb]
\centering
\caption{Number of parameters for different methods. Three networks with the minimal number of parameters are noted in red.}
    \begin{tabular}{ccccccc}
    \toprule
        Methods  &  DDNet & LDA & DuDoTrans & LEARN++ & LAMA-Net & iLAMA-Net \\
         \midrule
        \# Parameters & 6e5 & \red{6e4} & 8e6 & 6e6 & \red{3e5} & \red{4e5} \\
        \bottomrule
    \end{tabular}
\label{n_param}
\end{table}

\begin{table}[htb]
  \centering
  \caption{Quantitative results (AAPM-Mayo) by PSNR, SSIM, and Sinogram RMSE with 64 and 128 views. The RMSE is on the scale of $10^{-3}$. The best results in each column are in red.}
  \scalebox{0.7}{\parbox{\textwidth}{
  \begin{tabular}{ccccccc}
    \toprule
            \multirow{2}{*}{\textbf{Methods}}  &
            \multicolumn{3}{c}{$\*{64}$}  & \multicolumn{3}{c}{$\*{128}$} \\ &
            \multicolumn{1}{c}{PSNR} & \multicolumn{1}{c}{SSIM} & \multicolumn{1}{c}{RMSE ($10^{-3}$)} & \multicolumn{1}{c}{PSNR} & \multicolumn{1}{c}{SSIM} & \multicolumn{1}{c}{RMSE ($10^{-3}$)}\\
    \midrule
    FBP          & $27.17\pm1.11$ & $0.596\pm\expnum{9.0}{-4}$ & $19.32\pm\expnum{6.04}{-2}$ & $33.28\pm0.85$ & $0.759\pm\expnum{1.1}{-3}$ & $6.88\pm\expnum{1.57}{-4}$\\
    DDNet \cite{b8}       & $35.70\pm 1.50$ & $0.923\pm\expnum{3.9}{-4}$ & $7.53\pm\expnum{1.96}{-2}$ & $42.73\pm1.08$ & $0.974\pm\expnum{3.9}{-5}$ & $3.29\pm\expnum{1.31}{-3}$\\
    LDA \cite{b15}         & $37.16 \pm 1.33$ & $0.932\pm \expnum{2.0}{-4}$ & $3.31\pm\expnum{4.86}{-4}$ & $43.00 \pm 0.91$ & $0.976\pm\expnum{1.9}{-5}$ & $1.09\pm\expnum{7.26}{-6}$\\
    DuDoTrans \cite{b42}  & $37.90\pm1.44$ & $0.952\pm\expnum{1.0}{-4}$ & $2.42\pm\expnum{5.94}{-5}$ & $43.48\pm1.04$ & $0.985\pm\expnum{9.5}{-6}$ & $0.97\pm\expnum{7.51}{-6}$\\
    LEARN++ \cite{b14}    & $43.02\pm2.08$ & $0.980\pm\expnum{3.2}{-5}$ & $0.81\pm\expnum{2.43}{-5}$ & $49.77\pm0.96$ & $0.995\pm\expnum{1.1}{-6}$ & $0.31\pm\expnum{1.72}{-6}$\\
    LAMA-Net (Ours)   & $44.58\pm 1.15$ & $0.986\pm\expnum{7.1}{-6}$ & $0.68\pm\expnum{9.97}{-6}$ & $50.01 \pm 0.69$ & $0.995\pm\expnum{6.0}{-7}$ & $0.32\pm\expnum{1.66}{-6}$\\
    iLAMA-Net (Ours)    & $\red{46.37\pm 0.99}$ & $\red{0.990\pm\expnum{3.7}{-6}}$ & $\red{0.52\pm\expnum{4.68}{-6}}$ & $\red{51.02 \pm 0.63}$ & $\red{0.996\pm\expnum{3.7}{-7}}$ & $\red{0.27\pm\expnum{7.13}{-7}}$\\
    \bottomrule
  \end{tabular}
  }}
  \label{tab:comparison-TCIA}
\end{table}

\begin{table}[htb]
\centering
  \caption{Quantitative results (NBIA) by PSNR, SSIM, and Sinogram RMSE with 64 and 128 views. The RMSE is on the scale of $10^{-3}$. The best results in each column are in red.}
\scalebox{0.7}{\parbox{\textwidth}{
  \begin{tabular}{ccccccc}
    \toprule
    \multirow{2}{*}{\textbf{Methods}}  &
    \multicolumn{3}{c}{$\*{64}$}  & \multicolumn{3}{c}{$\*{128}$} \\ &
    \multicolumn{1}{c}{PSNR} & \multicolumn{1}{c}{SSIM} & \multicolumn{1}{c}{RMSE ($10^{-3}$)} & \multicolumn{1}{c}{PSNR} & \multicolumn{1}{c}{SSIM} & \multicolumn{1}{c}{RMSE ($10^{-3}$)}  \\
    \midrule
    FBP         & $25.72\pm1.93$ & $0.592\pm\expnum{1.6}{-3}$ & $29.02\pm0.26$ & $31.86\pm1.27$ & $0.743\pm\expnum{1.7}{-3}$ & $7.51\pm\expnum{4.60}{-3}$\\
    DDNet\cite{b8}       & $35.59\pm2.76$ & $0.920\pm\expnum{2.7}{-4}$ & $7.26\pm\expnum{2.68}{-2}$ & $40.23\pm1.98$ & $0.961\pm\expnum{1.2}{-4}$ & $1.45\pm\expnum{7.79}{-5}$\\
    LDA  \cite{b15}       & $34.31\pm2.20$ & $0.896\pm\expnum{3.9}{-4}$ & $6.43\pm\expnum{1.85}{-3}$ & $40.26\pm2.57$ & $0.963\pm\expnum{1.3}{-4}$ & $1.76\pm\expnum{6.93}{-5}$\\
    DuDoTrans \cite{b42}  & $35.53\pm2.63$ & $0.938\pm\expnum{2.4}{-4}$ & $2.40\pm\expnum{2.63}{-4}$ & $40.67\pm2.84$ & $0.976\pm\expnum{6.2}{-5}$ & $1.06\pm\expnum{4.21}{-5}$\\
    LEARN++  \cite{b14}   & $38.53\pm3.41$ & $0.956\pm\expnum{2.3}{-4}$ & $1.52\pm\expnum{1.13}{-4}$ & $43.35\pm4.02$ & $0.983\pm\expnum{5.3}{-5}$ & $0.73\pm\expnum{2.27}{-5}$\\
    LAMA-Net  (Ours)  & $41.40\pm 3.54$ & $0.976\pm\expnum{8.1}{-5}$ & $0.99\pm\expnum{3.54}{-5}$ & $45.20\pm4.23$ & $0.988\pm\expnum{3.0}{-5}$ & $0.55\pm\expnum{1.24}{-5}$\\
    iLAMA-Net (Ours)   & $\red{42.11\pm4.09}$ & $\red{0.979\pm\expnum{7.5}{-5}}$ & $\red{0.92\pm\expnum{1.79}{-5}}$ & $\red{47.28\pm5.13}$ & $\red{0.992\pm\expnum{1.5}{-5}}$ & $\red{0.44\pm\expnum{7.40}{-6}}$\\
    \bottomrule
  \end{tabular}}}
  \label{tab:comparison-NBIA}
\end{table}

\subsection{Comparison in reconstruction quality}
We evaluate LAMA-Net and iLAMA-Net by applying sparse-view sinograms and the ground truth images obtained by FBP from the test set. We test state-of-the-art (SOTA) methods: DDNet \cite{b8}, LDA \cite{b15}, DuDoTrans \cite{b42}, Learn++ \cite{b14}. 
\begin{figure}[htb]
\centering
\includegraphics[width=\linewidth]{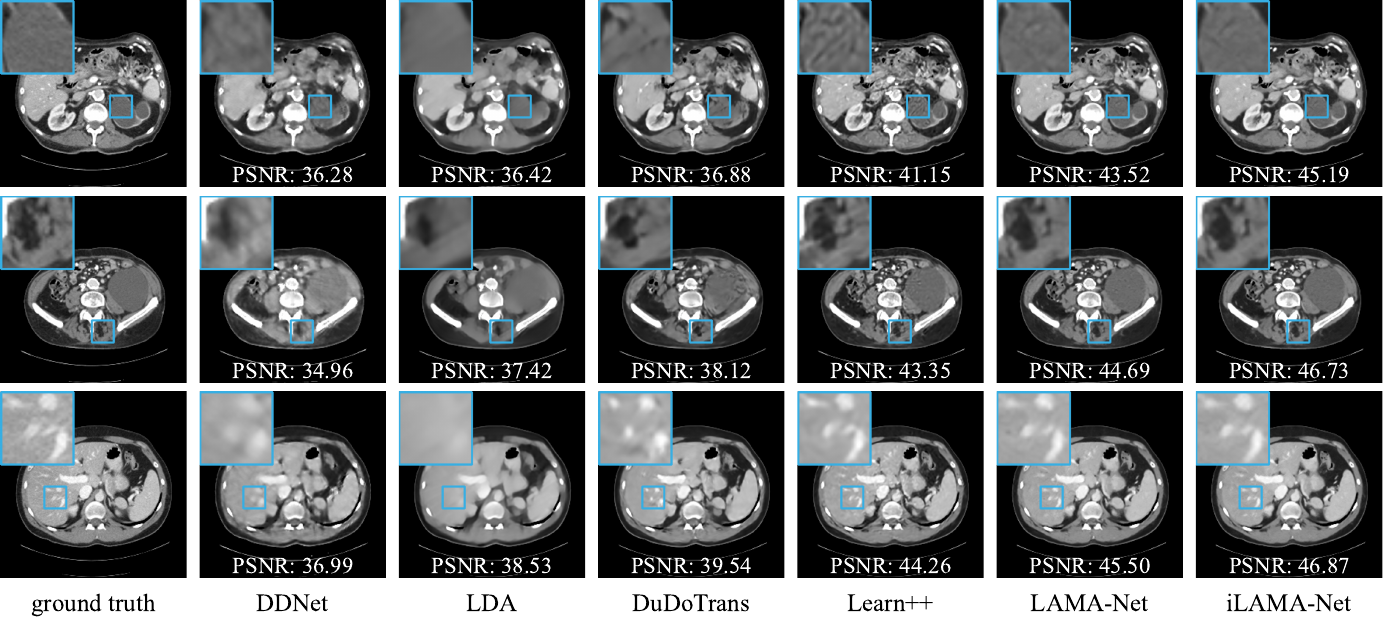}
\caption{Visual comparison using 64-view siongrams for AAPM-Mayo dataset.}
\label{fig:results-mayo}
\end{figure}

\begin{figure}[htb]
\centering

\includegraphics[width=\linewidth]{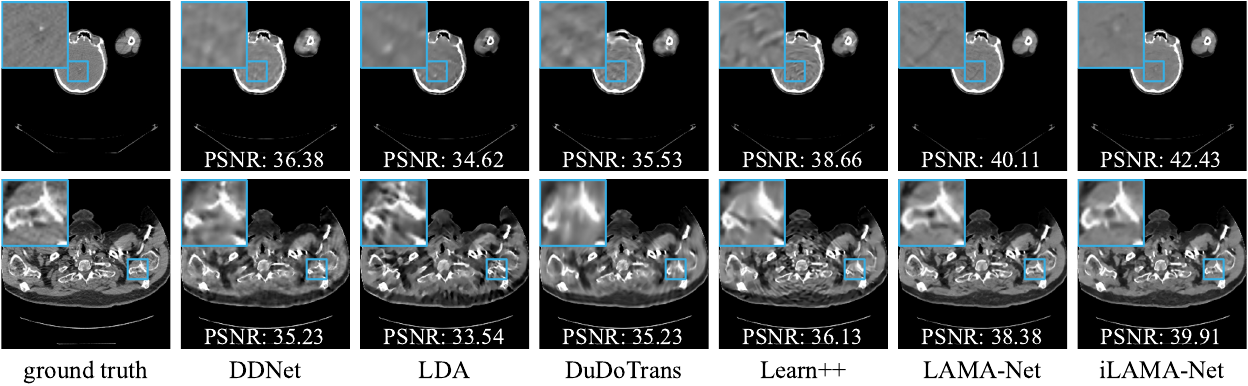}
\label{fig:results-NBIA}
\caption{Visual comparison using 64-view siongrams for NBIA dataset.}
\end{figure}
The reconstruction quality is evaluated and presented in Table \ref{tab:comparison-TCIA} for the AAPM-Mayo and NIBIA datasets, respectively. Note that we provide sinogram RMSE to demonstrate the effectiveness of LAMA as a dual-domain method that leverages complementary information.
Moreover, the network sizes of these compared networks are shown in Table \ref{n_param}. 
Example reconstructed images (with zoom-ins) of these methods on the two datasets are shown in Figure \ref{fig:results-mayo}.

From Tables \ref{tab:comparison-TCIA} and \ref{tab:comparison-NBIA}, we can see LAMA-Net performs favorably against all SOTA deep reconstruction networks: it attained the highest PSNR and SSIM among all compared methods. These indicate that LAMA-Net is highly accurate in image reconstruction after training on small-size datasets. LAMA-Net also produces the smallest sinogram RMSE, proving that our method effectively exploits information from both domains. Notably, the subtle decrease in sinogram RMSE with increasingly accurate reconstruction yields substantial improvements in PSNR and SSIM metrics, highlighting the pivotal role of refinements in the observation data, namely, the sinogram. This observation not only underscores the efficacy of LAMA-Net in recovering the original observation but also suggests a proximity to a globally optimum solution.  
Meanwhile, LAMA-Net also uses a magnitude of fewer parameters, which demonstrates its parameter effectiveness and simplicity for training.
Figure \ref{fig:results-mayo} further demonstrates that LAMA-Net can effectively preserve structural details in reconstructed images when removing noise and artifacts.

\begin{table}[htb]
\centering
\caption{Comparison for different initialization methods with or without LAMA-Net for AAPM-Mayo dataset.}
\scalebox{0.75}{\parbox{\textwidth}{
\begin{tabular}{cccccccc}
    \toprule
        \multirow{2}{*}{\textbf{Methods}}  &
            \multicolumn{3}{c}{$\*{64}$}  & \multicolumn{3}{c}{$\*{128}$} \\
            & PSNR & SSIM & RMSE ($10^{-3}$) & PSNR & SSIM & RMSE ($10^{-3}$)\\
         \midrule
        FBP & $27.17 \pm 1.11$ & $0.596\pm\expnum{9.0}{-4}$ & $19.32\pm\expnum{6.04}{-2}$ & $33.28\pm0.85$ & $0.759\pm\expnum{1.11}{-3}$ & $6.87\pm\expnum{1.57}{-3}$\\
        CNN & $34.43 \pm 1.55$ & $0.916\pm \expnum{3.0}{-4}$ & $2.95\pm\expnum{2.42}{-4}$ & $41.30\pm1.75$& $0.975\pm\expnum{3.77}{-5}$ & $1.03\pm\expnum{4.39}{-5}$\\
        Init-Net & $37.14 \pm 1.56$ & $0.942 \pm \expnum{1.6}{-4}$ & $1.95\pm\expnum{1.13}{-4}$ & $43.61\pm1.14$ & $0.984 \pm \expnum{1.20}{-5}$ & $0.75\pm\expnum{1.50}{-5}$\\\\

        FBP + LAMA-Net & $44.58\pm 1.15$ & $0.986\pm\expnum{7.1}{-6}$ & $0.68\pm\expnum{9.97}{-6}$ & $50.01 \pm 0.69$ & $0.995\pm\expnum{6.0}{-7}$ & $0.32\pm\expnum{1.66}{-6}$\\
        CNN + LAMA-Net & $45.11 \pm 1.03$ & $0.987\pm\expnum{6.6}{-6}$ & $0.62 \pm \expnum{7.16}{-6}$ & $50.69\pm0.68$ & $0.995 \pm \expnum{4.3}{-7}$ & $0.29\pm \expnum{1.03}{-6}$\\
        iLAMA-Net & $\red{46.37\pm 0.99}$ & $\red{0.990\pm\expnum{3.7}{-6}}$ & $\red{0.52\pm\expnum{4.68}{-6}}$ & $\red{51.02 \pm 0.63}$ & $\red{0.996\pm\expnum{3.7}{-7}}$ & $\red{0.27\pm\expnum{7.13}{-7}}$\\
        \bottomrule
    \end{tabular}}}
\label{tab:init}
\end{table}

As we mentioned earlier, LAMA-Net can benefit significantly from a good initial estimate, given that the original problem \eqref{eq:OrgPhi} is nonconvex and nonsmooth. There are various methods to generate such an initial estimate, including using FBP of partial sinogram data or inpainting techniques applied to the partial data. In Section \ref{subsect:init}, we introduced Init-Net specifically designed for sparse-view CT (SVCT) applications. Init-Net incorporates a geometry-driven design tailored to the structure of sparse-view sinograms.

We conducted experiments to compare the solution quality of the direct outputs from three initialization methods—FBP of sparse-view sinogram, CNN-based inpainting, and our proposed Init-Net—as well as their performance when used as initial inputs to LAMA-Net. The CNN consists of 5 sequential blocks with skip connection, each containing four convolutional layers with kernels of size 3 × 3, stride 1, and padding 1 × 1, where all layers use ReLU activation, requiring more parameters than Init-Net. The CNN inpainting initialization network is trained on image data pairs $\{(\text{FBP}(\*s_0^{(i)}), \hat{\mathbf{x}}^{(i)})\}_{i=1}^N$ from the same data set for 100 epochs using ADAM \cite{b46} with learning rate $10^{-4}$, producing refined images $\{\mathbf{x}_{\text{CNN}}^{(i)}\}_{i=1}^N$. The corresponding LAMA-Net with CNN as the initialization network is then trained in the same supervised manner, where $(\mathbf{x}_{K}^{(i)}, \mathbf{z}_{K}^{(i)}) \coloneqq \text{LAMA}(\mathbf{x}_0 \coloneqq \mathbf{x}_{\text{CNN}}^{(i)}, \mathbf{z}_0 \coloneqq \mathbf{A}\mathbf{x}_{\text{CNN}}^{(i)})$.
Further details about the network architectures and parameters are provided in Section \ref{subsect:lama-net}, and the results are summarized in Table \ref{tab:init}. As shown in the table, iLAMA-Net, when combined with our specially designed Init-Net, consistently outperforms all other methods across different accuracy metrics and numbers of views in SVCT. Our method without an initialization network is referred to as LAMA-Net in Table \ref{tab:comparison-TCIA} and \ref{tab:comparison-NBIA}. In Table \ref{tab:init}, the variant with FBP initialization is denoted as FBP+LAMA-Net. Note that FBP+LAMA-Net and LAMA-Net are the same methods, with identical results, and the naming distinction is solely to clarify the initialization strategy in the ablation study.

While FBP and CNN-based inpainting are inherently simpler approaches, they struggle with large angular gaps in the data and fail to effectively suppress artifacts. This limitation is particularly evident for the CNN-based method, even though it employs more convolutional blocks than Init-Net. These results highlight the advantage of Init-Net's geometry-driven design in achieving sharper and more accurate reconstructions, confirming its suitability for initializing sparse-view sinograms in conjunction with LAMA-Net.

\begin{figure*}[htb]
    \centering
        \subfloat[\label{fig:CUSI1}]{\includegraphics[width=\linewidth]{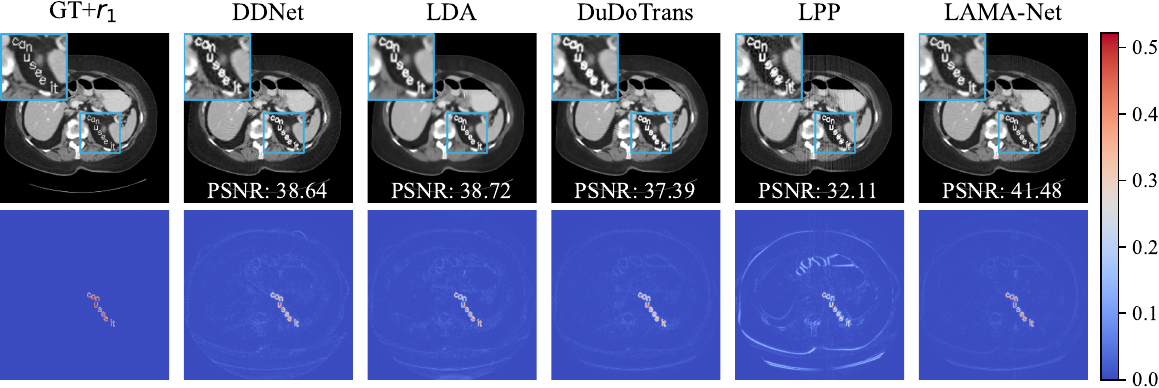}}\hfill
        \subfloat[\label{fig:CUSI2}]{\includegraphics[width=\linewidth]{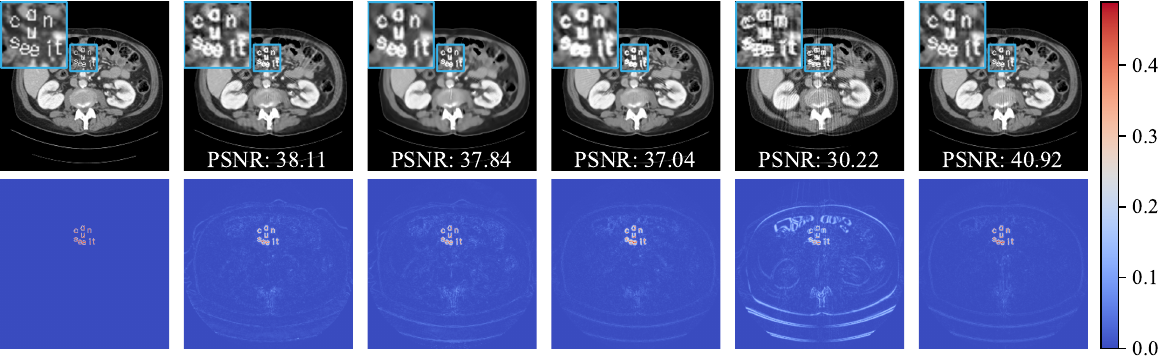}}
    \caption{Two examples of different algorithms using input with structured perturbation $r_1$. The first rows are reconstructed images with zoomed-in perturbed areas in the top left corner and annotated PSNR between the reconstructed image and perturbed ground truth GT+$r_1$. The second row shows the difference in heat maps between the reconstructed images and the unperturbed ground truth.}
    \label{fig:CUSI}
\end{figure*}

\subsection{Test on solution stability}

The goal of this section is to evaluate the stability of our algorithm against small perturbations, such as noise and structural changes (e.g., Gaussian noise or localized alterations like a tumor in a brain image). By stability, we refer to the ability of the method to accurately recover an image with the same perturbations that were initially introduced without imposing any additional unexpected local or global artifacts or noise. This aspect is particularly important in medical image reconstruction and deep learning applications \cite{b47}.

We follow the stability tests discussed in \cite{b47} and evaluate the compared methods under two types of perturbations: structured noise $r_1$ and Gaussian noise $r_2$. For the structured noise, we add the text ``can u see it" to highly structured areas in the ground truth CT images and then apply the same physical process to generate SVCT data. For Gaussian noise, they are added to the entire ground truth images.  Note that all compared methods, including LAMA-Net, are trained with clean data without applying any adversarial training or denoising strategy. The results are shown in Figures \ref{fig:CUSI} and \ref{fig:Gaussian}, respectively. 
\begin{figure}[htb]
    \centering
        \subfloat[\label{fig:Gaussian1}]{\includegraphics[width=\linewidth]{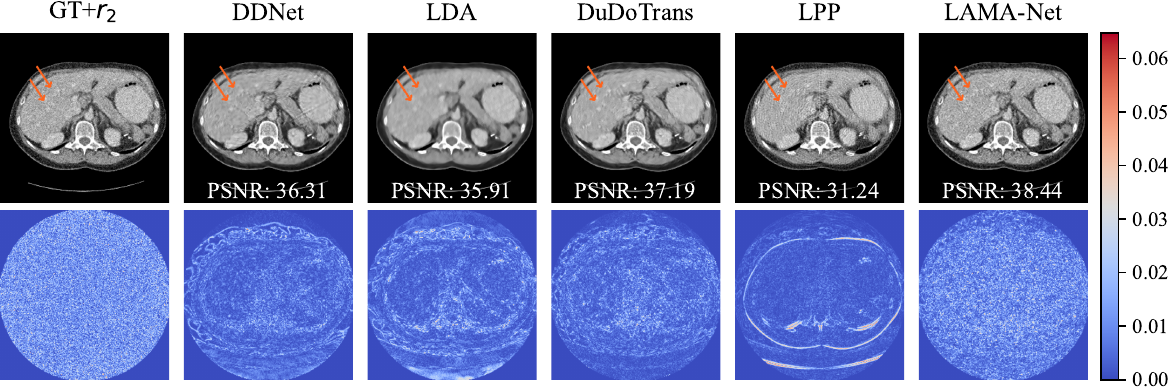}}
        \hfill
        \subfloat[\label{fig:Gaussian2}]{\includegraphics[width=\linewidth]{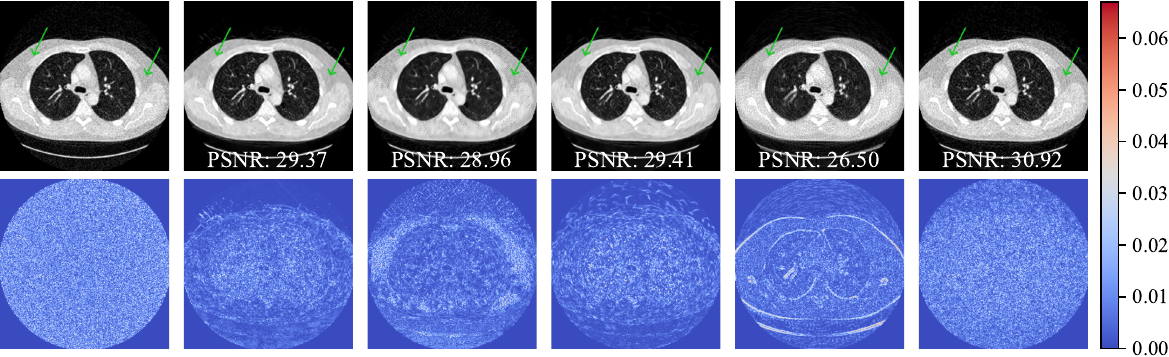}}
    \caption{Example results of reconstructed images with annotated PSNR between the reconstructed image and the perturbed ground truth GT+$r_2$, where $r_2$ is Gaussian noise. (a) $r_2$ is zero mean with a standard deviation of 0.03. (b) $r_2$ is zero mean with a standard deviation of 0.05.}
    \label{fig:Gaussian}
\end{figure}

These two figures show that LAMA-Net recovers the ground truth much more faithfully than other methods under perturbations $r_1$ and $r_2$.

%

\begin{figure}[htb]
    \centering
    \includegraphics[width=0.7\linewidth]{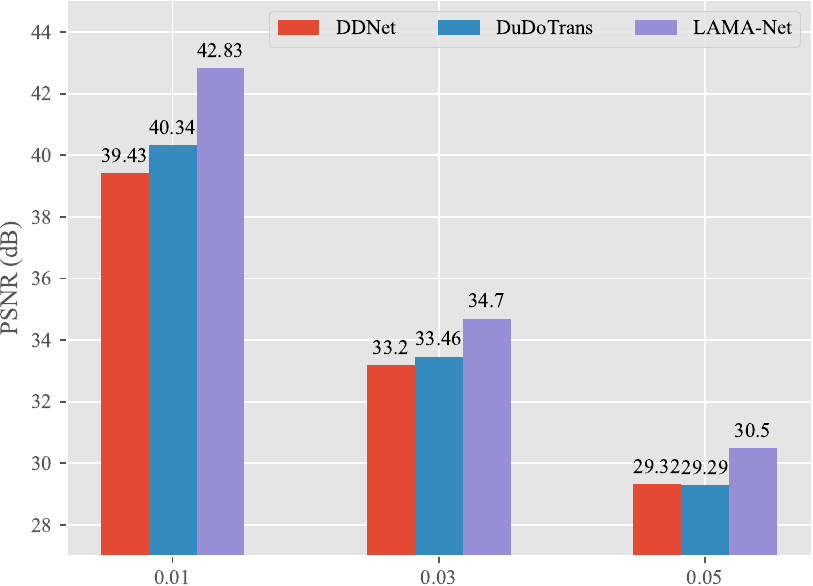}
    \caption{We pick the top-3 most stable algorithms (DDNet, DuDoTrans, LAMA-Net) and compare their reconstruction PSNR with perturbed ground truth by different Gaussian noise levels with standard deviation 0.01, 0.03 and 0.05, respectively, where the image pixel range is normalized between 0 and 1. }
    \label{fig:gaussian-comp}
\end{figure}

Our method demonstrates a notable capability in handling the structured noise $r_1$ ``can u see it". In Figure \ref{fig:CUSI}, the preservation of $r_1$ shows that LAMA-Net can treat outliers independently without introducing additional artifacts or instabilities.

This stability is evident in the faithfulness of features (mostly dark in the difference heatmap Figure \ref{fig:pert_reg} other than the perturbed area ``can u see it"), indicating the effectiveness of regularizers learned from unperturbed data in handling perturbations.
Notably, in both examples, LAMA-Net exhibits the smallest detail loss compared to other methods. This is evidenced by the noticeable retention of information, particularly highlighted by contours or edges in the difference heat map, which are clearly more apparent in other methods. Overall, other methods either show distortions not appearing in the original image or compromise small structures by over-smoothing. In contrast, delicate structures are still clearly visible in our proposed method. 
\begin{figure}[htb]
    \centering
    \includegraphics[width=0.8\linewidth]{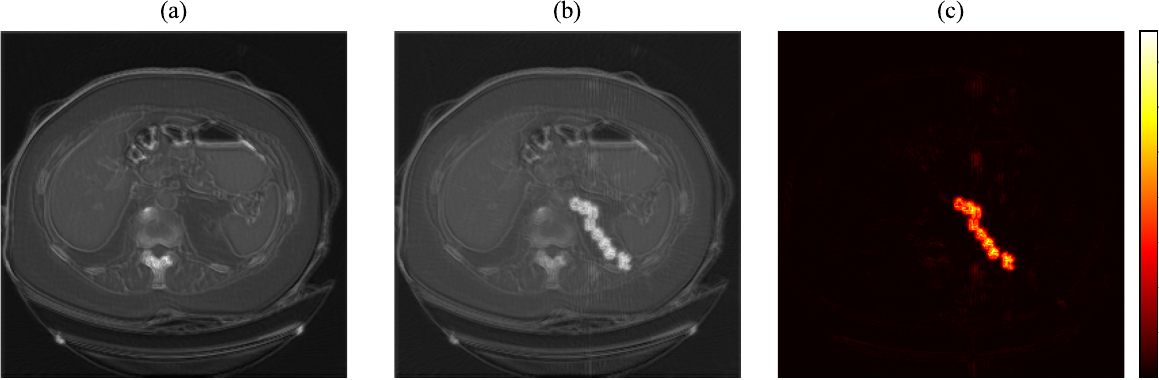}
 \caption{ (a) The $L^2$-norm of $g^R(\xbf)$ ($\norm{g^R(\xbf)}$), where the feature extractor $g^R$ is defined in Eq.~\eqref{eq:21norm} and learned from clean training data. (b) The $L^2$-norm of $g^R(\xbf+r_1)$ ($\norm{g^R(\xbf+r_1)}$), where $r_1$ is the perturbation added on $\xbf$. (c) The difference heatmap between (a) and (b). 
 }
    \label{fig:pert_reg}
\end{figure}
If the data is added with Gaussian noise $r_2$, note that in Figure \ref{fig:Gaussian}, both difference heat maps in the first column are the standard disk of Gaussian noise $r_2$ where no pattern or information related to the image can be observed. 
This characteristic is desirable as it signifies the algorithm's stability in effectively isolating and distinguishing noise from the essential information during reconstruction. In contrast, other methods such as DDNet, LDA, and DuDoTrans either tend to compromise small structures in Figure \ref{fig:Gaussian1} or excessively smooth certain areas of the reconstructed images in Figure \ref{fig:Gaussian2}, particularly evident in the areas pointed by the arrows. While this may enhance visual appeal, it poses a risk of information loss. This concern is further established by the difference heat map, revealing substantial disparity in the form of clearly visible highlighted contours and sharp edges, indicative of potential loss of critical details. LAMA also consistently produces more stable reconstructions for different noise levels than other methods shown in Figure \ref{fig:gaussian-comp}.

\section{Conclusion}
\label{sec:conclusion}

In this work, we provide a complete and rigorous convergence proof of LAMA developed in our prior work \cite{b2} and show that all accumulation points of a specified subsequence of LAMA must be Clarke stationary points of the dual-domain reconstruction problem. LAMA directly yields a highly interpretable neural network architecture called LAMA-Net. In addition, we demonstrate that the convergence property of LAMA indeed yields outstanding stability and robustness empirically. We also show that the performance of LAMA-Net can be further improved by integrating a properly designed network that generates suitable initials, which we call iLAMA-Net in short. To evaluate LAMA-Net/iLAMA-Net, we conduct several experiments and compare them with several state-of-the-art methods on popular benchmark datasets for Sparse-View Computed Tomography. The results of these experiments demonstrate the outstanding performance of LAMA-Net/iLAMA-Net when compared to other methods.

\section*{Acknowledgments}
This research is supported in part by National Science Foundation under grants DMS-2152960, DMS-2152961, DMS-2307466, and DMS-2409868.

\bibliographystyle{abbrv}
\bibliography{sn-bibliography}

\begin{thebibliography}{10}

\bibitem{b47}
V.~Antun, F.~Renna, C.~Poon, B.~Adcock, and A.~C. Hansen.
\newblock On instabilities of deep learning in image reconstruction and the potential costs of ai.
\newblock {\em Proceedings of the National Academy of Sciences}, 117(48):30088--30095, 2020.

\bibitem{b27}
J.~Bolte, S.~Sabach, and M.~Teboulle.
\newblock Proximal alternating linearized minimization for nonconvex and nonsmooth problems - mathematical programming, Jul 2013.

\bibitem{b13}
H.~Chen, Y.~Zhang, Y.~Chen, J.~Zhang, W.~Zhang, H.~Sun, Y.~Lv, P.~Liao, J.~Zhou, and G.~Wang.
\newblock Learn: Learned experts' assessment-based reconstruction network for sparse-data ct.
\newblock {\em IEEE Transactions on Medical Imaging}, 37(6):1333--1347, 2018.

\bibitem{b16}
H.~Chen, Y.~Zhang, M.~K. Kalra, F.~Lin, Y.~Chen, P.~Liao, J.~Zhou, and G.~Wang.
\newblock Low-dose ct with a residual encoder-decoder convolutional neural network.
\newblock {\em IEEE Transactions on Medical Imaging}, 36(12):2524--2535, 2017.

\bibitem{b15}
Y.~Chen, H.~Liu, X.~Ye, and Q.~Zhang.
\newblock Learnable descent algorithm for nonsmooth nonconvex image reconstruction.
\newblock {\em SIAM Journal on Imaging Sciences}, 14(4):1532--1564, 2021.

\bibitem{b38}
I.~Chun, Z.~Huang, H.~Lim, and J.~A. Fessler.
\newblock Momentum-net: Fast and convergent iterative neural network for inverse problems.
\newblock {\em IEEE Transactions on Pattern Analysis; Machine Intelligence}, 45(04):4915--4931, apr 2023.

\bibitem{b35}
I.~Y. Chun and J.~A. Fessler.
\newblock Convolutional analysis operator learning: Acceleration and convergence.
\newblock {\em IEEE Transactions on Image Processing}, 29:2108--2122, 2019.

\bibitem{liu2023omni}
Y.~Cui, W.~Ren, and A.~Knoll.
\newblock Omni-kernel network for image restoration.
\newblock {\em Proceedings of the AAAI Conference on Artificial Intelligence}, 38(2):1426--1434, Mar. 2024.

\bibitem{lin2023selective}
Y.~Cui, Y.~Tao, Z.~Bing, W.~Ren, X.~Gao, X.~Cao, K.~Huang, and A.~Knoll.
\newblock Selective frequency network for image restoration.
\newblock In {\em The Eleventh International Conference on Learning Representations}, 2023.

\bibitem{b2}
C.~Ding, Q.~Zhang, G.~Wang, X.~Ye, and Y.~Chen.
\newblock Learned alternating minimization algorithm for dual-domain sparse-view ct reconstruction.
\newblock In {\em Medical Image Computing and Computer Assisted Intervention -- MICCAI 2023}, pages 173--183, Cham, 2023. Springer Nature Switzerland.

\bibitem{b26}
R.~Ge, Y.~He, C.~Xia, H.~Sun, Y.~Zhang, D.~Hu, S.~Chen, Y.~Chen, S.~Li, and D.~Zhang.
\newblock Ddpnet: A novel dual-domain parallel network for low-dose ct reconstruction.
\newblock In L.~Wang, Q.~Dou, P.~T. Fletcher, S.~Speidel, and S.~Li, editors, {\em Medical Image Computing and Computer Assisted Intervention -- MICCAI 2022}, pages 748--757, Cham, 2022. Springer Nature Switzerland.

\bibitem{b36}
A.~Gkillas, D.~Ampeliotis, and K.~Berberidis.
\newblock Connections between deep equilibrium and sparse representation models with application to hyperspectral image denoising.
\newblock {\em IEEE Transactions on Image Processing}, 32:1513--1528, 2023.

\bibitem{b9}
K.~Hornik, M.~Stinchcombe, and H.~White.
\newblock Multilayer feedforward networks are universal approximators.
\newblock {\em Neural networks}, 2(5):359--366, 1989.

\bibitem{b17}
K.~H. Jin, M.~T. McCann, E.~Froustey, and M.~Unser.
\newblock Deep convolutional neural network for inverse problems in imaging.
\newblock {\em IEEE Transactions on Image Processing}, 26(9):4509--4522, 2017.

\bibitem{b46}
D.~P. Kingma and J.~Ba.
\newblock Adam: A method for stochastic optimization, 2017.

\bibitem{b5}
H.~Lee, J.~Lee, H.~Kim, B.~Cho, and S.~Cho.
\newblock Deep-neural-network-based sinogram synthesis for sparse-view ct image reconstruction.
\newblock {\em IEEE Transactions on Radiation and Plasma Medical Sciences}, 3(2):109--119, 2018.

\bibitem{b10}
S.~Liang and R.~Srikant.
\newblock Why deep neural networks for function approximation?
\newblock In {\em International Conference on Learning Representations (ICLR)}, 2017.

\bibitem{lukic2024moment}
T.~Luki\'{c} and P.~Bal\'{a}zs.
\newblock Moment preserving tomographic image reconstruction model.
\newblock {\em Image Vision Comput.}, 146(C), June 2024.

\bibitem{lukic2016binary}
T.~Lukić and P.~Balázs.
\newblock Binary tomography reconstruction based on shape orientation.
\newblock {\em Pattern Recognition Letters}, 79:18--24, 2016.

\bibitem{lukic2022limited}
T.~Lukić and P.~Balázs.
\newblock Limited-view binary tomography reconstruction assisted by shape centroid.
\newblock {\em The Visual Computer}, 38(2):695--705, 2022.

\bibitem{b44}
B.~D. Man and S.~Basu.
\newblock Distance-driven projection and backprojection in three dimensions.
\newblock {\em Physics in Medicine and Biology}, 49(11):2463--2475, May 2004.

\bibitem{b12}
V.~Monga, Y.~Li, and Y.~C. Eldar.
\newblock Algorithm unrolling: Interpretable, efficient deep learning for signal and image processing.
\newblock {\em IEEE Signal Processing Magazine}, 38(2):18--44, 2021.

\bibitem{b34}
S.~Mukherjee, A.~Hauptmann, O.~{\"O}ktem, M.~Pereyra, and C.-B. Sch{\"o}nlieb.
\newblock Learned reconstruction methods with convergence guarantees: a survey of concepts and applications.
\newblock {\em IEEE Signal Processing Magazine}, 40(1):164--182, 2023.

\bibitem{b40}
Y.~Nesterov.
\newblock Smooth minimization of non-smooth functions.
\newblock {\em Mathematical Programming}, 103(1):127--152, Dec 2004.

\bibitem{b30}
Y.~Nesterov.
\newblock {\em Lectures on Convex Optimization}.
\newblock Springer Publishing Company, Incorporated, 2nd edition, 2018.

\bibitem{b28}
T.~Pock and S.~Sabach.
\newblock Inertial proximal alternating linearized minimization (ipalm) for nonconvex and nonsmooth problems.
\newblock {\em SIAM Journal on Imaging Sciences}, 9(4):1756--1787, 2016.

\bibitem{b29}
B.~Polyak.
\newblock Some methods of speeding up the convergence of iteration methods.
\newblock {\em USSR Computational Mathematics and Mathematical Physics}, 4(5):1--17, 1964.

\bibitem{b37}
Y.~Sun, Z.~Wu, X.~Xu, B.~Wohlberg, and U.~S. Kamilov.
\newblock Scalable plug-and-play admm with convergence guarantees.
\newblock {\em IEEE Transactions on Computational Imaging}, 7:849--863, 2021.

\bibitem{b42}
C.~Wang, K.~Shang, H.~Zhang, Q.~Li, Y.~Hui, and S.~K. Zhou.
\newblock Dudotrans: Dual-domain transformer provides more attention for sinogram restoration in sparse-view ct reconstruction, 2021.

\bibitem{b43}
Z.~Wang, A.~Bovik, H.~Sheikh, and E.~Simoncelli.
\newblock Image quality assessment: from error visibility to structural similarity.
\newblock {\em IEEE Transactions on Image Processing}, 13(4):600--612, 2004.

\bibitem{b24}
W.~Wu, D.~Hu, C.~Niu, H.~Yu, V.~Vardhanabhuti, and G.~Wang.
\newblock Drone: Dual-domain residual-based optimization network for sparse-view ct reconstruction.
\newblock {\em IEEE Transactions on Medical Imaging}, 40(11):3002--3014, 2021.

\bibitem{b45}
W.~Xia, Z.~Lu, Y.~Huang, Z.~Shi, Y.~Liu, H.~Chen, Y.~Chen, J.~Zhou, and Y.~Zhang.
\newblock Magic: Manifold and graph integrative convolutional network for low-dose ct reconstruction.
\newblock {\em IEEE Transactions on Medical Imaging}, 40(12):3459--3472, 2021.

\bibitem{b20}
W.~Xia, Z.~Yang, Q.~Zhou, Z.~Lu, Z.~Wang, and Y.~Zhang.
\newblock A transformer-based iterative reconstruction model for sparse-view ct reconstruction.
\newblock In L.~Wang, Q.~Dou, P.~T. Fletcher, S.~Speidel, and S.~Li, editors, {\em Medical Image Computing and Computer Assisted Intervention -- MICCAI 2022}, pages 790--800, Cham, 2022. Springer Nature Switzerland.

\bibitem{b22}
L.~Yang, Z.~Li, R.~Ge, J.~Zhao, H.~Si, and D.~Zhang.
\newblock Low-dose ct denoising via sinogram inner-structure transformer.
\newblock {\em IEEE Transactions on Medical Imaging}, 42(4):910--921, 2023.

\bibitem{b8}
D.~Yarotsky.
\newblock Error bounds for approximations with deep relu networks.
\newblock {\em Neural Networks}, 94:103--114, 2017.

\bibitem{b23}
L.~Yu, Z.~Zhang, X.~Li, and L.~Xing.
\newblock Deep sinogram completion with image prior for metal artifact reduction in ct images.
\newblock {\em IEEE Transactions on Medical Imaging}, 40(1):228--238, 2021.

\bibitem{b11}
J.~Zhang and B.~Ghanem.
\newblock Ista-net: Interpretable optimization-inspired deep network for image compressive sensing.
\newblock In {\em 2018 IEEE/CVF Conference on Computer Vision and Pattern Recognition}, pages 1828--1837, 2018.

\bibitem{b19}
J.~Zhang, Y.~Hu, J.~Yang, Y.~Chen, J.-L. Coatrieux, and L.~Luo.
\newblock Sparse-view x-ray ct reconstruction with gamma regularization.
\newblock {\em Neurocomputing}, 230:251--269, 2017.

\bibitem{b14}
Y.~Zhang, H.~Chen, W.~Xia, Y.~Chen, B.~Liu, Y.~Liu, H.~Sun, and J.~Zhou.
\newblock Learn++: Recurrent dual-domain reconstruction network for compressed sensing ct.
\newblock {\em IEEE Transactions on Radiation and Plasma Medical Sciences}, 7(2):132--142, 2023.

\bibitem{b25}
B.~Zhou, X.~Chen, H.~Xie, S.~K. Zhou, J.~S. Duncan, and C.~Liu.
\newblock Dudoufnet: Dual-domain under-to-fully-complete progressive restoration network for simultaneous metal artifact reduction and low-dose ct reconstruction.
\newblock {\em IEEE Transactions on Medical Imaging}, 41(12):3587--3599, 2022.

\end{thebibliography}

\end{document}